\documentclass{article} %
\usepackage{iclr2027_conference,times}

\usepackage{xspace}
\newcommand{\method}{Prospective Forcing\xspace}

\usepackage{amsmath,amsfonts,bm}

\def\eqref#1{equation~\ref{#1}}

\def\1{\bm{1}}

\DeclareMathAlphabet{\mathsfit}{\encodingdefault}{\sfdefault}{m}{sl}
\SetMathAlphabet{\mathsfit}{bold}{\encodingdefault}{\sfdefault}{bx}{n}

\usepackage{wrapfig}
\usepackage{booktabs}   %
\usepackage{multirow}   %
\usepackage{graphicx}   %
\usepackage{graphicx} 
\usepackage{hyperref}
\usepackage{url}
\usepackage{amsmath,amssymb,amsfonts}
\usepackage{amsthm}
\usepackage{enumitem}
\usepackage{mathrsfs}
\newtheorem{theorem}{Theorem}
\newtheorem{lemma}{Lemma}

\title{Prospective Forcing: Learning to Look Ahead for Autoregressive Long Video Diffusion}

\author{Yifan Shen$^{1,3\dagger}$ \ Yewei Xia$^{1}$ \ Gongxu Luo$^{1,3}$ \ Peiyuan Zhu$^{1}$ \ Boyang Sun$^{1}$ \\
\bf Zhenhao Chen$^{1}$ \ Guangyi Liu$^{1,3}$ \ Guangyi Chen$^{1,2*}$ \ Kun Zhang$^{2,1}$ \\[1pt]
{\normalfont\footnotesize $^{1}$Mohamed bin Zayed University of AI \ \ $^{2}$Carnegie Mellon University \ \ $^{3}$Institute of Foundation Models}
}

\iclrfinalcopy
\begin{document}

\maketitle
{\renewcommand{\thefootnote}{}\footnotetext{$^{\dagger}$Work done during an internship at the Institute of Foundation Models. \quad $^{*}$Corresponding author.}}

\begin{figure}[h]
\centering
\vspace{-0.9cm}
\includegraphics[width=\textwidth]{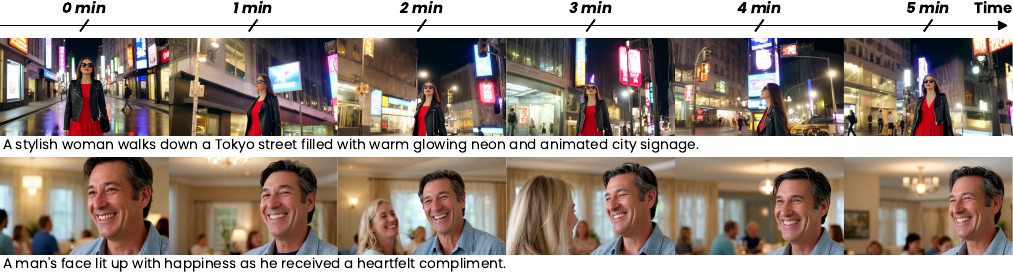}
\vspace{-0.4cm}
\caption{By learning to predict future, Prospective Forcing generates long videos that remain dynamic, temporally consistent, and visually stable from a single text prompt.}
\label{fig:teaser}
\vspace{-0.1cm}
\end{figure}

\begin{abstract}
Autoregressive video generators synthesize long videos through chunk-wise rollouts, yet they are typically trained to predict only the immediate future. This limited prediction horizon makes it challenging to maintain video quality over extended rollouts. To address this problem, we introduce Prospective Forcing, a multi-head prediction framework that injects long-range predictive signals into autoregressive video representations. Built on a rolling-window generator, Prospective Forcing attaches lightweight, offset-specific heads that share the same conditioning information and independently predict chunks at different future temporal offsets. By extending supervision beyond the next chunk, these heads encourage the shared history representation to retain information useful for longer-horizon prediction, supporting temporally consistent generation. At inference time, we accelerate sampling by using accepted future-chunk predictions to skip full-window backbone executions. We equip the head with learnable embeddings to generate multiple candidates for the next chunk, and a learned verifier either selects a candidate or rejects all and falls back to standard autoregressive generation.
\end{abstract}
\vspace{-12pt}

\section{Introduction}
\label{sec:introduction}

\begin{quote}
\itshape
\small
``The future enters into us, in order to transform itself in us,
long before it happens.''

\raggedleft
\normalfont
\small
--- Rainer Maria Rilke, \textit{Letters to a Young Poet}
\end{quote}

Autoregressive video generation provides a natural framework for modeling dynamic visual worlds by extending a video one chunk at a time, which supports generative interactive environments~\citep{pmlr-v235-bruce24a}, real-time game simulation~\citep{ICLR2025_b71ecea2}, robot control~\citep{li2026causalworldmodelingrobot}, and other applications that require coherent long-term visual rollouts. However, maintaining temporal consistency over long videos remains an open challenge. A rollout may remain locally plausible while subject identity, background appearance, spatial layout, color, and object structure gradually drift. These failures become more visible as the generation proceeds. This drift stems from a lack of long-horizon predictive capability, as next-chunk supervision prioritizes immediate continuation without requiring representations to preserve the information needed to anticipate more distant futures. We refer to this myopic behavior as \emph{temporal shortsightedness}.

Existing approaches improve the temporal consistency from different perspectives. Conditioning the model on self-generated histories during training makes later predictions more robust to an imperfectly generated context~\citep{huang2025selfforcingbridgingtraintest}. Rolling-window generation jointly denoises adjacent chunks with progressive noise levels and retains initial states as global anchors, reducing error accumulation while preserving a reference to the beginning of the video~\citep{liu2025rolling}. Training-free methods extend pretrained video generators by rescheduling noise across overlapping temporal windows or maintaining a queue of partially denoised frames~\citep{DBLP:conf/iclr/QiuX0HWS024,DBLP:conf/nips/KimKCH24}. Other approaches combine short-term conditioning on features from the preceding chunk with long-term preservation of scene and object features from the initial chunk~\citep{DBLP:conf/cvpr/HenschelKPHTWNS25}. 

These approaches primarily address exposure bias, inference-time context preservation, or temporal denoising schedules. We introduce \method to address temporal shortsightedness through a complementary training objective. Specifically, the backbone encodes preceding chunks into a shared history representation, and independent heads combine it with target-specific diffusion noise to predict chunks at distinct future offsets. We train these heads through cross-chunk self-distillation using the backbone's own subsequent outputs as the teacher. When a target chunk is reached later in the same causal rollout, the backbone’s detached feature and velocity predictions supervise the corresponding longer-offset prediction from an earlier context. These losses propagate through both the heads and the backbone, encouraging the shared representation to retain information needed across multiple future chunks. The resulting representation captures rapid changes in local motion while preserving relatively stable properties such as object identity and scene structure, consistent with the fast-and-slow perspective of hierarchical causal representation learning~\citep{NEURIPS2025_1c6decac}. We call it a \emph{speculative representation} to reflect its long-term predictive capability. Section~\ref{sec:theory} further analyzes this representation in an idealized setting.

At inference, our default configuration uses the speculative representation to produce multiple candidate continuations. A lightweight learned verifier selects one candidate or rejects them all, in which case generation falls back to the standard autoregressive path. Each accepted candidate extends the video by one additional chunk without another sequential execution of the backbone, reducing the sampling cost. We evaluate on 30-second video generation with multiple metrics, where \method improves video quality, while accelerating sampling. Notably, the improvements persist even when the prediction heads are removed at inference, indicating that the speculative representation itself captures information that benefits long-horizon generation.

Our main contributions are as follows. First, we introduce \method to address temporal shortsightedness by learning a speculative representation that supports explicit prediction of multiple future chunks from the current state, thereby improving long-term video quality. Second, we develop a verifier-guided speculative inference procedure that reuses the prediction heads to extend rollouts without additional sequential backbone executions. Third, we demonstrate improved video quality across multiple video generation metrics alongside sampling speedups.

\section{Related Work}

\label{sec:related_work}

\textbf{Autoregressive video generation.}
Unlike bidirectional video diffusion models which jointly denoise a fixed temporal window~\citep{blattmann2023stablevideodiffusion,yang2024cogvideox,kong2024hunyuanvideo,wan2025wan}, autoregressive models generate frames or chunks sequentially, enabling low-latency streaming with KV caching. Existing methods differ mainly in how the causal generator is trained: with per-token noise levels~\citep{chen2024diffusionforcing}, by distilling few-step causal students from bidirectional or causal teachers via distribution matching~\citep{yin2024dmd,yin2025causvid,huang2025selfforcingbridgingtraintest,zhu2026causalforcing,zhao2026causal}, or end-to-end without a teacher~\citep{guo2025resamplingforcing}.

\textbf{Long video generation.}
To mitigate drift over long rollouts, prior methods either align training with long-horizon inference~\citep{cui2026selfforcingplusplus,yang2026longlive,chen2026contextforcing,lin2026dynamem}, design denoising schedules and anchor frames that preserve early context~\citep{liu2025rolling,DBLP:conf/cvpr/HenschelKPHTWNS25}, or manage the KV cache and positional encoding at inference time without additional training~\citep{yi2026deepforcing,cui2026lol}. Instead, our method improves long video generation by learning representations with long-horizon predictive capabilities.

\textbf{Future-predictive representations.}
Predicting future observations has long served as an auxiliary objective for learning sequential representations in reinforcement learning~\citep{pmlr-v119-guo20g,schwarzer2021spr}, image generation~\citep{yu2026mirai}, language modeling and speculative decoding~\citep{pmlr-v235-gloeckle24a,li2024eagle,cai2024medusa}. Concurrent video works such as Video-Mirai~\citep{yu2026videomirai} and Next Forcing~\citep{xu2026nextforcing} incorporate future information via JEPA-style representation alignment~\citep{,bardes2024vjepa} and recurrent future-chunk prediction, respectively. In contrast, \method learns a shared representation for parallel multi-chunk prediction with independent heads, and reuses these heads for verifier-guided speculative inference. A more detailed discussion is provided in Appendix~\ref{app:related_work}.

\section{Preliminaries}

\label{sec:preliminaries}

\paragraph{Autoregressive video generation.} We formulate long-video generation as autoregressive prediction over latent video chunks. Let $\mathbf{x}_{1:T}=\{\mathbf{x}_1,\ldots,\mathbf{x}_T\}$ denote a sequence of video chunks and $\mathbf{z}_{1:T}=\{\mathbf{z}_1,\ldots,\mathbf{z}_T\}$ their corresponding latents encoded by a pretrained VAE. Given a text condition $\mathbf{c}$, a causal generator factorizes the conditional distribution as
\begin{equation}
p_\theta\left(\mathbf{z}_{1:T}\mid\mathbf{c}\right) = \prod_{t=1}^{T} p_\theta\left(\mathbf{z}_t\mid\mathbf{z}_{<t},\mathbf{c}\right),
\label{eq:ar_factorization}
\end{equation}
where $\mathbf{z}_{<t}=\{\mathbf{z}_1,\ldots,\mathbf{z}_{t-1}\}$ denotes the preceding latent chunks and $t$ indexes chunks in the autoregressive rollout. This factorization specifies the causal dependence among committed video chunks, but does not require each chunk to be denoised in isolation. We use $\tau\in[0,1]$ to denote diffusion time. Under the rectified-flow parameterization~\citep{liu2023rectified}, the forward corruption process is
\begin{equation}
\mathbf{z}_t(\tau) = (1-\tau)\mathbf{z}_t+\tau\boldsymbol{\epsilon}_t, \qquad \boldsymbol{\epsilon}_t\sim\mathcal{N}(\mathbf{0},\mathbf{I}).
\label{eq:forward}
\end{equation}
A velocity-prediction network transforms noisy latents toward clean samples while conditioning on the text prompt and the available causal context. Sampling evolves from Gaussian noise at $\tau=1$ toward clean latents at $\tau=0$, with the resulting video chunks progressively committed to the autoregressive context. The key--value states of previously committed chunks can be cached to avoid repeated computation.

\paragraph{Distribution matching distillation.} CausVid~\citep{yin2025causvid} distills a pretrained bidirectional video diffusion model into a few-step causal generator using distribution matching distillation (DMD)~\citep{yin2024dmd}. Let $\hat{\mathbf{z}}\sim p_\theta(\cdot\mid\mathbf{c})$ denote a latent video sampled from the causal generator, and let $\hat{\mathbf{z}}(\tau)$ denote its noisy version obtained using~\eqref{eq:forward}. Let $\theta^\star$ denote the frozen parameters of the teacher model. DMD trains the generator to match the teacher model by minimizing the reverse KL divergence between their noise-perturbed sample distributions:
\begin{equation}
\mathcal{L}_{\mathrm{DMD}} = \mathbb{E}_{\mathbf{c},\tau}\left[w(\tau)\,D_{\mathrm{KL}}\left(p_{\theta,\tau}(\cdot\mid\mathbf{c})\,\|\,p_{\theta^\star,\tau}(\cdot\mid\mathbf{c})\right)\right],
\label{eq:dmd}
\end{equation}
where $p_{\theta,\tau}$ and $p_{\theta^\star,\tau}$ denote the distributions of noisy generator and teacher samples, respectively. Here, $w(\tau)$ is a timestep-dependent weighting function. The gradient of this objective is estimated using the difference between the generator and teacher scores, where a score function is the gradient of the log probability density with respect to the noisy latent.

\begin{wrapfigure}{r}{0.61\linewidth}
  \vspace{-12pt}
  \centering
  \includegraphics[width=\linewidth]{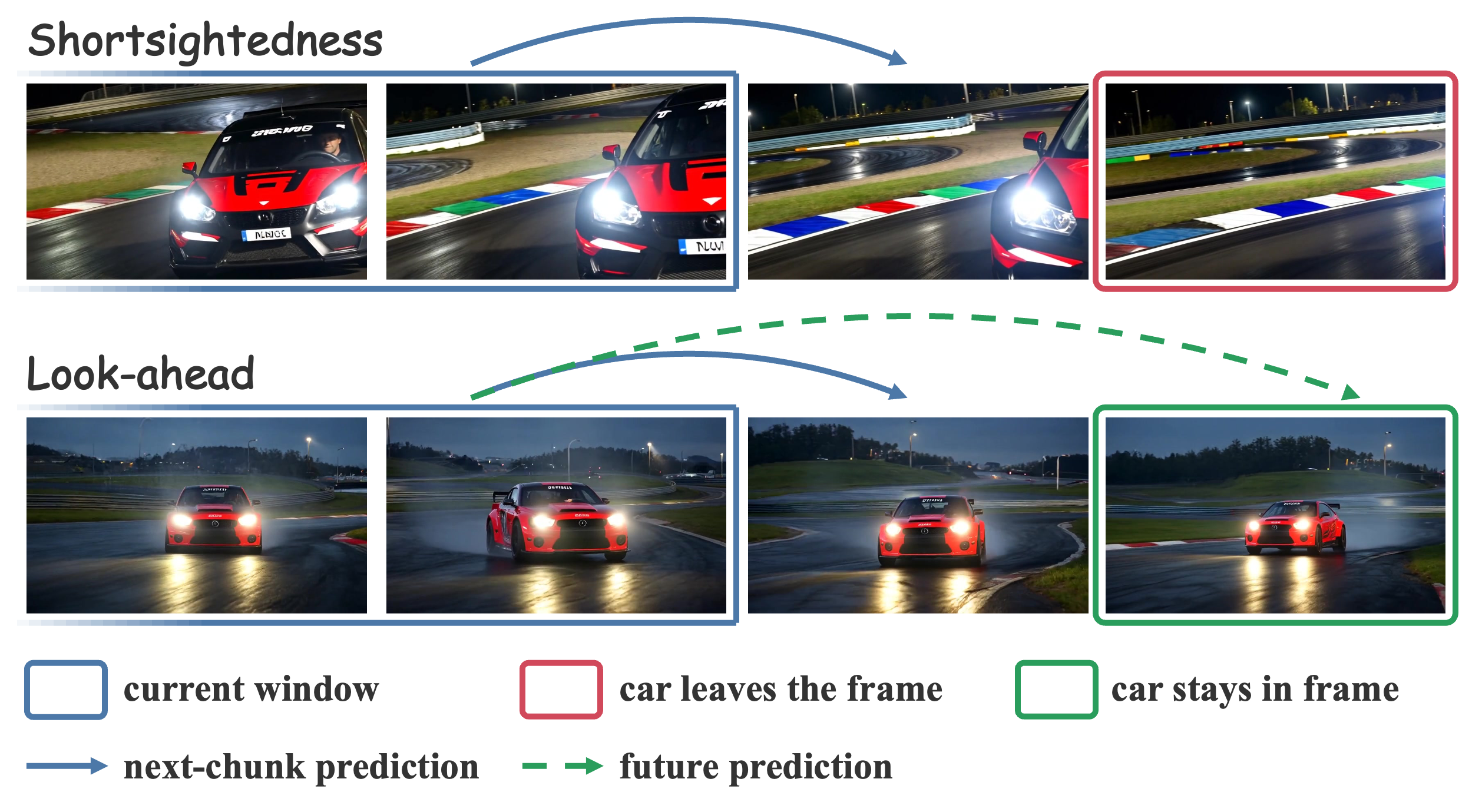}
  \vspace{-16pt}
  \caption{\textbf{Temporal shortsightedness.} The car gradually drives out of view with only next-chunk prediction, and foreseeing future chunks keeps it in frame.}
  \label{fig:shortsightedness}
  \vspace{-10pt}
\end{wrapfigure}

To bridge the train--test distribution gap, Self Forcing~\citep{huang2025selfforcingbridgingtraintest} extends this framework with autoregressive self-rollout training and a DMD loss computed on the entire generated sequence. Rolling Forcing~\citep{liu2025rolling} further jointly denoises multiple chunks at different noise levels~\citep{chen2024diffusionforcing} within a rolling window to enhance temporal consistency. However, the resulting generator still lacks an explicit mechanism for anticipating future chunks, a limitation we term \emph{temporal shortsightedness}, as shown in Figure~\ref{fig:shortsightedness}.

\section{Prospective Forcing}

\label{sec:prospective_forcing}

\begin{figure}[t]
  \centering
  \includegraphics[width=\linewidth]{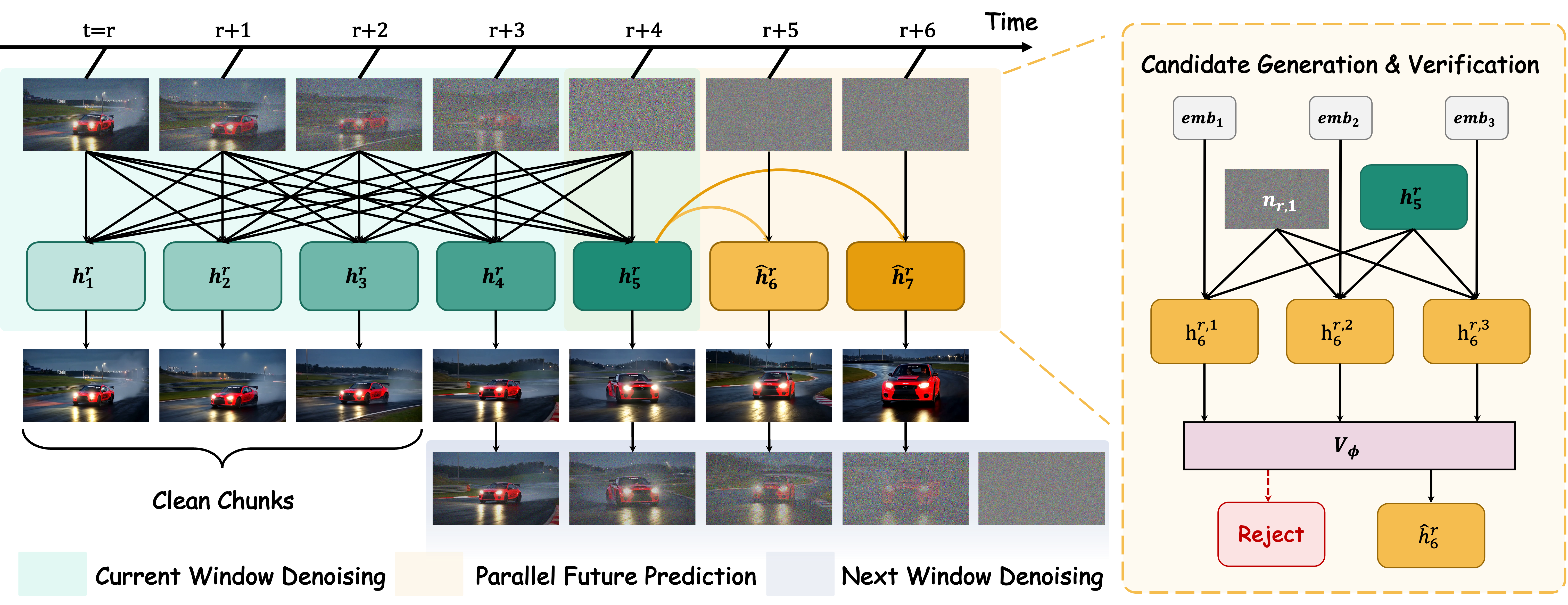}
   \vspace{-0.5cm}
  \caption{\textbf{Illustration of Prospective Forcing.} \textbf{Left}: We jointly denoise $W=5$ chunks in window $r$, yielding features $\mathbf{h}_1^r$ to $\mathbf{h}_5^r$ (green). Then $K=2$ heads predict $\hat{\mathbf{h}}_6^r$ and $\hat{\mathbf{h}}_7^r$ in parallel from $\mathbf{h}_5^r$ (orange). Leading clean chunks are committed and the window advances accordingly (blue). \textbf{Right}: $M=3$ candidates are generated for the next chunk $\hat{h}_6^r$, and the verifier $V_\phi$ selects one or rejects all.}
  \label{fig:overview}
  \vspace{-0.3cm}
\end{figure}

We propose \emph{Prospective Forcing}, which mitigates temporal shortsightedness by learning a speculative representation for future prediction (Figure~\ref{fig:overview}). Within a rolling window, we jointly denoise $W$ chunks at different noise levels, yielding features $\mathbf{h}^r_1, \dots, \mathbf{h}^r_W$ from the backbone. $K$ lightweight future prediction heads then predict the features of the next $K$ chunks, $\hat{\mathbf{h}}^r_{W+1}, \dots, \hat{\mathbf{h}}^r_{W+K}$, in parallel from $\mathbf{h}^r_W$, and a shared projection network maps them to velocities. For the next chunk, the first head further generates $M$ candidates, each conditioned on a distinct learned candidate embedding, and a verifier $V_\phi$ is trained to score them. If the best candidate satisfies the acceptance criterion, it is accepted; otherwise, all candidates are rejected. The accepted prediction is used together with the outputs in the window to produce clean chunks. Finally, the preceding clean chunks are committed, and the last $W-1$ chunks are renoised to form the window for the next iteration.

\textbf{Speculative representation learning.} Our generator processes consecutive latent chunks within a rolling window of $W$ chunks. Let $\mathcal{W}_r$ denote the window whose first chunk is at time step $r$, so that it covers chunks $r,\ldots,r+W-1$. Throughout this section, a superscript $r$ indicates that a quantity is computed in window $\mathcal{W}_r$, and a subscript $i\in\{1,\ldots,W\}$ indicates a position within the window, corresponding to chunk $r+i-1$. Each position is assigned a fixed diffusion time from $\boldsymbol{\tau}=(\tau_1,\ldots,\tau_W)$, with earlier positions at lower noise levels. Let $\mathbf{z}_i^r$ denote the noisy latent obtained by applying ~\eqref{eq:forward} to chunk $r+i-1$ at diffusion time $\tau_i$. The window is then 
$\mathcal{W}_r = \left(\mathbf{z}_1^r,\ldots,\mathbf{z}_W^r\right).$
We denote the generator backbone by $F_\theta$, and represent the cached context from previously committed chunks by $\mathcal{C}_r$. Conditioned on the text prompt $\mathbf{c}$, the backbone produces features for all positions in the window:
\begin{equation}
\left(\mathbf{h}_1^r,\ldots,\mathbf{h}_W^r\right) = F_\theta\left(\mathcal{W}_r,\boldsymbol{\tau},\mathbf{c},\mathcal{C}_r\right),
\label{eq:window_features}
\end{equation}
where $\mathbf{h}_i^r$ denotes the final layer features at position $i$ of $\mathcal{W}_r$. The shared velocity projection $P_\theta$ maps the features at each position to the corresponding velocity output $\mathbf{v}_i^r$:
\begin{equation}
\mathbf{v}_i^r = P_\theta\left(\mathbf{h}_i^r,\tau_i\right), \qquad i=1,\ldots,W.
\label{eq:velocity_projection}
\end{equation}

We use the backbone features of the newest chunk, $\mathbf{h}_W^r$, as the shared input to $K$ lightweight future prediction heads $G_k$, indexed by $k=1,\ldots,K$. Each head estimates features for chunk $r+W+k-1$, the $k$-th chunk after the window, which we index as position $W+k$. We sample the Gaussian noise $\boldsymbol{\epsilon}_{r+W+k-1}$ of these future chunks in advance and embed it into tokens $\mathbf{n}_{r,k}$ with the patch embedding of the backbone. The predicted features $\hat{\mathbf{h}}_{W+k}^r$ and corresponding velocities $\hat{\mathbf{v}}_{W+k}^r$ are obtained as
\begin{equation}
\begin{aligned}
\hat{\mathbf{h}}_{W+k}^r &= G_k\left(\mathbf{h}_W^r,\mathbf{n}_{r,k}\right),\quad
\hat{\mathbf{v}}_{W+k}^r &= P_\theta\left(\hat{\mathbf{h}}_{W+k}^r,\tau_W\right).
\end{aligned}
\label{eq:future_prediction}
\end{equation}
The shared projection $P_\theta$ decodes each prediction at $\tau_W$, the diffusion time assigned to an incoming chunk. The pre-sampled noise is reused when the corresponding chunk enters the window, ensuring that the prediction and its later target refer to the same noise realization.

We align these predictions with the corresponding backbone outputs in subsequent windows using cross-chunk self-distillation. After $k$ advances, chunk $r+W+k-1$ enters $\mathcal{W}_{r+k}$ at its last position $W$, where the backbone produces its features $\mathbf{h}_W^{r+k}$ and velocity $\mathbf{v}_W^{r+k}$. We use both outputs as supervision targets and stop gradients through them. Note that we do not skip any window at training time, so the supervision can always be provided. We normalize features before alignment to reduce sensitivity to their scale. Velocities are matched in their original space because both their direction and magnitude determine the denoising updates. For two feature tensors $\mathbf{a}$ and $\mathbf{b}$, we define a distance $d_{\mathrm{feat}}$ that combines mean squared error and cosine distance after layer normalization:
\begin{equation}
d_{\mathrm{feat}}(\mathbf{a},\mathbf{b}) ={} \frac{1}{2}\operatorname{MSE}\left(\operatorname{LN}(\mathbf{a}),\operatorname{LN}(\mathbf{b})\right)
+ 1-\operatorname{Cos}\left(\operatorname{LN}(\mathbf{a}),\operatorname{LN}(\mathbf{b})\right),
\label{eq:feature_distance}
\end{equation}
where $\operatorname{LN}$ denotes layer normalization~\citep{ba2016layer} applied along the feature dimension of each token. $\operatorname{MSE}$ averages squared errors over all elements, and $\operatorname{Cos}$ averages cosine similarity over corresponding tokens. Let $\operatorname{sg}$ denote the stop gradient operator. The feature and velocity alignment losses are
\begin{equation}
\small
\mathcal{L}_{\mathrm{feat}} = \mathbb{E}\left[\frac{1}{K}\sum_{k=1}^{K}d_{\mathrm{feat}}\left(\hat{\mathbf{h}}_{W+k}^r,\operatorname{sg}\left[\mathbf{h}_W^{r+k}\right]\right)\right],\quad
\mathcal{L}_{\mathrm{vel}} = \mathbb{E}\left[\frac{1}{K}\sum_{k=1}^{K}\operatorname{MSE}\left(\hat{\mathbf{v}}_{W+k}^r,\operatorname{sg}\left[\mathbf{v}_W^{r+k}\right]\right)\right],
\label{eq:future_alignment}
\end{equation}
where the expectation is taken over training rollouts. These alignment losses encourage $\mathbf{h}_W^r$ to encode information useful for predicting future chunks. We refer to these learned features as the \emph{speculative representation}.

\textbf{Candidate generation and verification.} We further equip the first future prediction head $G_1$ with candidate generation. With learned embeddings $\mathbf{emb}_m$ that are broadcast on tokens and added to noise $\mathbf{n}_{r,1}$, $G_1$ produces $M$ candidate features $\hat{\mathbf{h}}_{W+1}^{r,m}$ and velocities $\hat{\mathbf{v}}_{W+1}^{r,m}$:
\begin{equation}
\hat{\mathbf{h}}_{W+1}^{r,m} = G_1\left(\mathbf{h}_W^{r},\mathbf{n}_{r,1}+\mathbf{emb}_m\right),\quad
\hat{\mathbf{v}}_{W+1}^{r,m} = P_\theta\left(\hat{\mathbf{h}}_{W+1}^{r,m},\tau_W\right).
\label{eq:candidate_prediction}
\end{equation}
We supervise the candidate $m_r^\star$ closest to the backbone features $\mathbf{h}_W^{r+1}$ computed in the next window $\mathcal{W}_{r+1}$. Using the feature distance defined in~\eqref{eq:feature_distance}, the candidate alignment loss is
\begin{equation}
\mathcal{L}_{\mathrm{cand}} = (1-\beta)\, d_{\mathrm{feat}}\left(\hat{\mathbf{h}}_{W+1}^{r,m_r^\star},\operatorname{sg}\left[\mathbf{h}_W^{r+1}\right]\right)
+ \frac{\beta}{M-1}\sum_{m\neq m_r^\star} d_{\mathrm{feat}}\left(\hat{\mathbf{h}}_{W+1}^{r,m},\operatorname{sg}\left[\mathbf{h}_W^{r+1}\right]\right),
\label{eq:candidate_alignment}
\end{equation}
where $m_r^\star = \arg\min_{m} E_m^r$ is the candidate closest to the backbone target under the normalized feature error $E_m^r$ in~\eqref{eq:candidate_error}. Since the closest candidate receives most of the gradient, the candidate embeddings are not pulled toward a common point and can specialize to different regions of the prediction space. The small weight $\beta$ keeps every candidate activated.

We further train a lightweight multilayer perceptron (MLP) verifier $V_\phi$ to select one candidate or reject them all. Given the current features $\mathbf{h}_W^{r}$ and all candidate features $\hat{\mathbf{h}}_{W+1}^{r,m}$, $V_\phi$ predicts selection probabilities $\pi_1^r,\ldots,\pi_M^r$ and a rejection probability $\pi_{\mathrm{rej}}^r$, with gradients stopped through its inputs so that the verifier loss does not update the generator. The training label $y_r$ is $m_r^\star$, or $\mathrm{rej}$ if even this closest candidate exceeds an error threshold $\delta$ (see Appendix~\ref{app:verifier} for details). The verifier loss is
\begin{equation}
\mathcal{L}_{\mathrm{ver}} = \mathbb{E}\left[-\log \pi_{y_r}^{r}+\left(\pi_{\mathrm{rej}}^r-\mathbf{1}[y_r=\mathrm{rej}]\right)^2\right],
\label{eq:verifier_loss}
\end{equation}
where $\mathbf{1}[\cdot]$ is the indicator function. At inference, we take the most probable class and accept the corresponding candidate, or fall back to standard inference if it is $\mathrm{rej}$.

\textbf{Training loss.} We train the model in two stages. Let $\lambda_{\mathrm{feat}}$, $\lambda_{\mathrm{vel}}$, and $\lambda_{\mathrm{aux}}$ denote the loss weights. In the first stage, we combine DMD with the future alignment losses:
\begin{equation}
\mathcal{L}_{\mathrm{stage1}} = \mathcal{L}_{\mathrm{DMD}}+\lambda_{\mathrm{feat}}\mathcal{L}_{\mathrm{feat}}+\lambda_{\mathrm{vel}}\mathcal{L}_{\mathrm{vel}}.
\label{eq:stage1_objective}
\end{equation}
In the second stage, we add the candidate alignment and verifier losses. For heads with candidates, $\mathcal{L}_{\mathrm{feat}}$ and $\mathcal{L}_{\mathrm{vel}}$ are computed from candidate $m_r^\star$. The loss function is then
\begin{equation}
\mathcal{L}_{\mathrm{stage2}} = \mathcal{L}_{\mathrm{stage1}}+\lambda_{\mathrm{aux}}\left(\mathcal{L}_{\mathrm{cand}}+\mathcal{L}_{\mathrm{ver}}\right).
\label{eq:training_objective}
\end{equation}

\textbf{Speculative inference.} In standard rolling inference, each backbone evaluation advances the window by only one chunk. A new chunk enters the last position as pure noise and must be processed once by the backbone before it can move forward. Prospective Forcing lets the future prediction head take this step in advance, since the head is trained to match the output that the backbone will produce for the same chunk in a future step under the same pre-sampled noise. 

At each step, the backbone processes $\mathcal{W}_r$ while $G_1$ proposes $M$ candidates for the upcoming chunk and $V_\phi$ evaluates them. An accepted candidate is used directly as an additional clean chunk within the same backbone evaluation, similar to Medusa~\citep{cai2024medusa}; unlike Medusa, verification uses a lightweight learned verifier rather than an extra backbone pass. If all candidates are rejected, the step reduces to a standard rolling step. Either way, the leading chunks are committed, the remaining $W-1$ chunks are re-noised to their new positions, and a new chunk with its pre-sampled noise enters the last position, so accepted and rejected steps can be freely interleaved. In the worst case, the method reduces to standard autoregressive generation with negligible overhead. 

To better maintain temporal consistency during generation, the KV cache $\mathcal{C}_r$ retains the key--value states of the first chunk as global context and the most recent committed chunk, which immediately precedes the window, as local context. We update the latter whenever the window advances.

\vspace{-0.2cm}
\section{Theoretical Guarantee}
\vspace{-0.1cm}

\label{sec:theory}

The alignment losses train $\mathbf{h}^r_W$ to anticipate the features and velocities of future chunks. We now characterize what this guarantees about the learned representation in an idealized single-chunk window ($W=1$).
We analyze the video chunks $\mathbf x_t$, each decoded from its latent $\mathbf z_t$, and assume they are generated from a first-order Markov state process $\{\mathbf s_t\}$, with $\mathbf x_t$ depending on $\mathbf s_{t:t-L}$. $\mathbf{h}^r_W$ then plays the role of an estimated state $\widehat{\mathbf s}_t=\widehat{\mathbf m}(\mathbf x_{t:t-L})$, from which the heads predict the future window $\mathbf x_{t+2L:t+L}$ under pre-sampled noise, which for $L=1$ covers exactly the offsets $\{1,2\}$ used in our experiments.
Letting $\vartheta$ collect the backbone and head parameters, distribution matching loss in~\eqref{eq:dmd} makes the estimated model observationally equivalent to the distribution it matches, which plays the role of the true process here.
Future alignment enters through the selection indicator $S_\alpha(\vartheta)=\mathbf 1[\mathcal L_{\mathrm{feat}}(\vartheta)+\mathcal L_{\mathrm{vel}}(\vartheta)\le\alpha]$: exact alignment ($\alpha=0$) forces the future window to depend on the past only through $\widehat{\mathbf s}_t$ (Lemma~\ref{a:bottleneck}), which is the conditional independence our argument requires of the estimated model.
Under the following conditions, this speculative representation recovers the true state up to a smooth invertible transformation.

\begin{theorem}[Block-wise identifiability]
\label{theorem:block-identifiability}
Let $\mathcal S=\prod_{k=1}^{n}I_k$ and $\widehat{\mathcal S}=\prod_{i=1}^{n}\widehat I_i$, products of open intervals in $\mathbb R^n$, be the true and estimated state spaces, and let $d_x$ be the dimension of a chunk. Let $\mathcal U,\mathcal Y\subseteq\mathbb R^{(L+1)d_x}$ be open domains for the recovery window $\mathbf x_{t:t-L}$ and the predictive window $\mathbf x_{t+2L:t+L}$, respectively. Suppose that the true process is time homogeneous, that its states $\{\mathbf s_t\}$ form a first-order Markov process, that each $\mathbf x_t$ depends only on $\mathbf s_{t:t-L}$, and that the chunks are conditionally independent given the states. Suppose further that:
\begin{enumerate}
\item (Smooth and positive density)
All state marginal densities and the transition density are strictly positive on their respective domains.
For some $t_0$, $\mathbf x_{t_0:t_0-L}$ has a Lebesgue density positive almost everywhere on $\mathcal U$.
For every $\mathbf{s}\in\mathcal{S}$, the conditional distribution of $\mathbf x_{t+2L:t+L}$ given $\mathbf{s}_t=\mathbf{s}$ admits a jointly Borel Lebesgue density.
\label{a:regularity}

\item (Smooth finite-context recovery)
There is a surjective $C^3$ map $\mathbf m:\mathcal U\to\mathcal S$ such that $\mathbf{s}_t=\mathbf m(\mathbf{x}_{t:t-L})$ almost surely for every $t$, and its Jacobian has full row rank: $\operatorname{rank}D\mathbf m(\mathbf u)=n$ for every $\mathbf u\in\mathcal U$.
\label{a:recovery}

\item (Predictive density separation)
For any distinct states $\mathbf{s},\mathbf{s}'\in\mathcal{S}$, the set $\{\mathbf y\in\mathcal Y: p(\mathbf y|\mathbf{s})\ne p(\mathbf y|\mathbf{s}')\}$ has positive probability under the marginal distribution of $\mathbf x_{t+2L:t+L}$.
\label{a:predictive-separation}
\end{enumerate}
Suppose the estimated model $(\widehat{\mathbf m},\widehat p)$, generated by the backbone and prediction heads with $W=1$, satisfies conditions~\ref{a:regularity}--\ref{a:predictive-separation} with $(\widehat{\mathbf m},\widehat p,\widehat{\mathcal S})$ in place of $(\mathbf m,p,\mathcal S)$, is time homogeneous, is observationally equivalent to the true model, and satisfies $S_0(\vartheta)=1$. Then there is a fixed $C^3$ diffeomorphism $\varphi:\widehat{\mathcal{S}}\to\mathcal{S}$ such that $\mathbf m(\mathbf u)=\varphi(\widehat{\mathbf m}(\mathbf u))$ for every $\mathbf u\in\mathcal U$. Hence $\mathbf{s}_t=\varphi(\widehat{\mathbf{s}}_t)$ almost surely, simultaneously for all $t$.
\end{theorem}

\paragraph{Discussion on assumptions.}
The assumptions combine identification conditions without imposing parametric constraints.
Smoothness and positivity provide regularity for the proof.
Predictive density separation parallels Assumption~4 of~\citep{hu2008instrumental}: distinct states must induce distinct future distributions.
Following~\citep{chen2024caring}, we assume that the state can be recovered from a finite observation window; smoothness and full rank of the recovery maps support identification through a smooth invertible transformation.
Two remarks delimit the scope of this result. The true process is the distribution matched by the training objective; under DMD this is the teacher's distribution rather than the data distribution. Moreover, the regularity conditions on the estimated model are assumed rather than derived, the selection is exact ($\alpha=0$), and the analysis covers $W=1$ without the KV-cache anchor. We do not formally relate identifiability to drift. However, future alignment consistently reduces drift and improves video quality empirically.

\vspace{-0.2cm}
\section{Experiments}

\vspace{-0.2cm}

\subsection{Setup}
\vspace{-0.1cm}

\textbf{Implementation.}
We build on Wan2.1-T2V-1.3B~\citep{wan2025wan}, initialized from the causal consistency distillation checkpoint of Causal Forcing++~\citep{zhao2026causal} and distilled from Wan2.1-T2V-14B, with $W=5$, $K=2$, and $M=2$. Training details are given in Appendix~\ref{app:implementation}.

\textbf{Evaluation.}
Following \citet{liu2025rolling}, we generate 30-second videos for the same 200 prompts from Movie Gen Video Bench~\citep{polyak2024movie} and report VBench~\citep{huang2024vbench}, VBench-Long~\citep{huang2025vbenchpp}, VideoScore~\citep{he2024videoscore}, MUSIQ~\citep{ke2021musiq} drift, and FPS. Please refer to Appendices~\ref{app:metrics} and~\ref{app:efficiency} for more details. We compare against six autoregressive generators built on the same base model: CausVid~\citep{yin2025causvid}, Self Forcing~\citep{huang2025selfforcingbridgingtraintest}, Causal Forcing++~\citep{zhao2026causal}, LongLive~\citep{yang2026longlive}, Rolling Forcing~\citep{liu2025rolling}, and Video-Mirai~\citep{yu2026videomirai}. We use the released checkpoint for Video-Mirai and retrain the others with their official code.

\vspace{-0.1cm}
\subsection{Main Results}
\vspace{-0.1cm}
\label{sec:main_results}

\begin{wrapfigure}[16]{r}{0.5\textwidth}
\centering
\vspace{-0.6cm}
\includegraphics[width=0.5\textwidth]{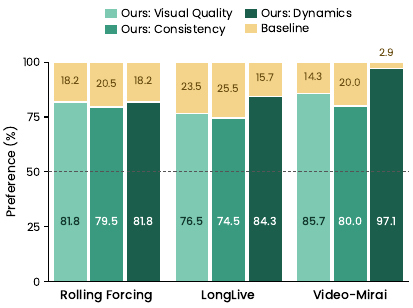}
\vspace{-0.7cm}
  \caption{\textbf{User study.} Bars show the win rate of pairwise comparisons of human preference.}
\label{fig:user_study}
\end{wrapfigure}
Figure~\ref{fig:qualitative_prompt12} illustrates the advantage of Prospective Forcing. Several baselines lose visual fidelity over time, such as the overexposed close-up of CausVid and the color corruption of Video-Mirai, or fail to keep the subject in view, as in Rolling Forcing. Others keep the car intact but gradually become static: in Self Forcing and LongLive, the car stays at almost the same position and pose relative to the camera, as if driving in place. In contrast, our method keeps the same car throughout the 30 seconds while it actually travels along the track and through the turn, with the viewpoint and surroundings changing accordingly. It demonstrates that Prospective Forcing generates long videos that are visually faithful, temporally consistent, and rich in motion. Human evaluation reflects the same trend. In our user study as shown in Figure~\ref{fig:user_study}, participants prefer our method over Rolling Forcing, LongLive, and Video-Mirai in visual quality, long-term consistency, and dynamics. Quantitatively, our method achieves the highest VBench Total and Quality, VBench-Long Quality, and MUSIQ scores in Table~\ref{tab:main}. LongLive attains slightly better drift and VideoScore, but with much less motion (Dynamic Degree 62.0 vs.\ 100.0, Table~\ref{tab:app_vbench_quality}); since static videos inflate consistency metrics, our method achieves a better balance between motion and consistency, as the user study also confirms.

\begin{figure*}[t]
\centering
\vspace{-0.3cm}
\includegraphics[width=0.9\textwidth]{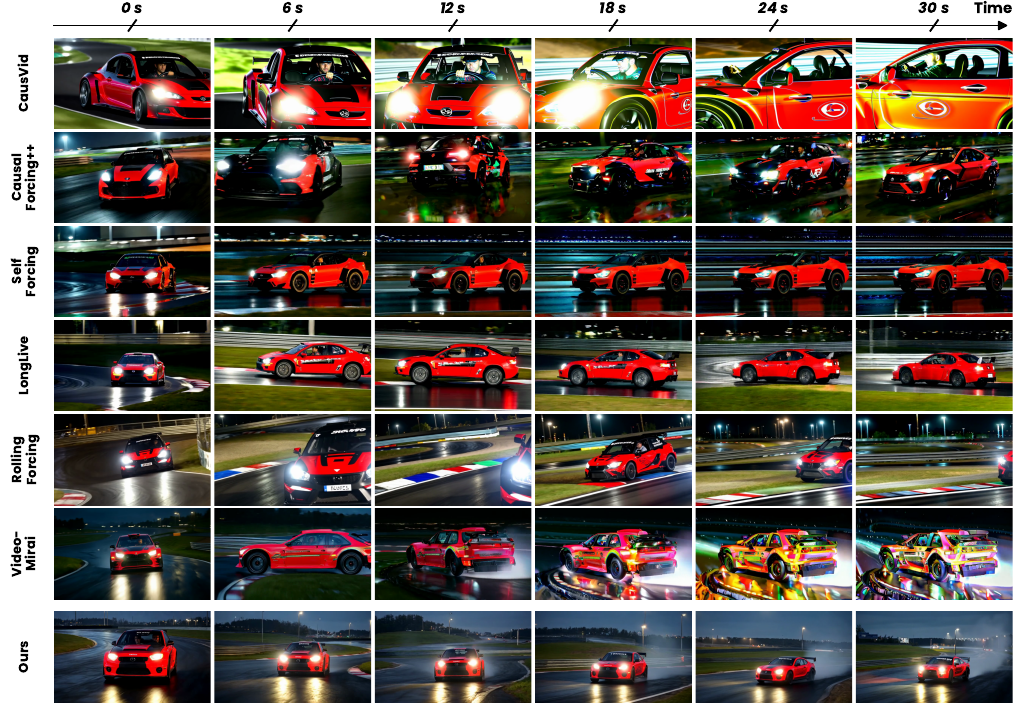}
\vspace{-0.2cm}
\caption{Qualitative comparison on a 30-second video generated from the prompt ``A rally car taking a fast turn on a track''. Video comparison is provided in the supplementary materials.}
\label{fig:qualitative_prompt12}
\end{figure*}

\begin{table*}[t]
\centering
\vspace{-0.3cm}
\caption{Quantitative comparison on 200 MovieGen prompts (30-second videos, $832\times480$, 16\,fps). \textbf{Bold} and \underline{underline} denote the best and second-best results.}
\vspace{0.1cm}
\label{tab:main}
\setlength{\tabcolsep}{6pt}
\renewcommand{\arraystretch}{1.12}
\resizebox{\textwidth}{!}{%
\begin{tabular}{l c ccc c c cc}
\toprule
\multirow{2}{*}{Method} & \multirow{2}{*}{FPS$\uparrow$} & \multicolumn{3}{c}{VBench} & VBench-Long & VideoScore & \multicolumn{2}{c}{MUSIQ} \\
\cmidrule(lr){3-5}\cmidrule(lr){6-6}\cmidrule(lr){7-7}\cmidrule(lr){8-9}
 & & Total$\uparrow$ & Quality$\uparrow$ & Semantic$\uparrow$ & Quality$\uparrow$ & Avg.$\uparrow$ & Avg.$\uparrow$ & Drift$\downarrow$ \\
\midrule
CausVid & \textbf{9.28} & 76.22 & 77.02 & 73.00 & 77.43 & 2.637 & 66.49 & 9.02 \\
Causal Forcing++ & 7.56 & 78.56 & 79.14 & 76.24 & 79.53 & 2.670 & 66.65 & 5.52 \\
Self Forcing & 7.79 & 78.27 & 78.76 & \underline{76.33} & 78.24 & 2.905 & 68.84 & 6.18 \\
LongLive & 8.08 & \underline{80.27} & \underline{80.78} & \textbf{78.22} & 80.61 & \textbf{3.108} & \underline{69.12} & \textbf{3.43} \\
Rolling Forcing & 7.95 & 78.63 & 80.05 & 72.94 & \underline{82.20} & 2.775 & 68.92 & 5.32 \\
Video-Mirai & 8.13 & 74.98 & 75.14 & 74.32 & 76.29 & 2.704 & 64.54 & 11.90 \\
\midrule
\textbf{Ours} & \underline{9.23} & \textbf{81.40} & \textbf{82.71} & 76.15 & \textbf{83.71} & \underline{3.069} & \textbf{69.14} & \underline{3.64} \\
\bottomrule
\end{tabular}}
\vspace{-0.3cm}
\end{table*}

\subsection{Ablation Study}

\textbf{Training and inference.}
Table~\ref{tab:abl_inference} compares inference modes under the two training stages. Even with standard inference ($\ell=0$), the stage-1 model already outperforms most baselines across multiple metrics and is preferred over them in most user study comparisons as shown in Table~\ref{tab:app_user_study_wo}. The learned representation also supports accelerated inference. Committing the predictions of $G_1,\dots,G_\ell$ without verification speeds up sampling by $1.56\times$ and $1.89\times$ for $\ell=1$ and $\ell=2$, while quality is largely preserved. In exchange, drift grows with $\ell$. Stage 2 further introduces candidate generation and the learned verifier. The resulting model achieves the highest quality, and its drift stays close to that of standard inference while still running $1.20\times$ faster. To disentangle stage-2 training from verification, we also evaluate the stage-2 model with standard and unverified inference. Both perform close to their stage-1 counterparts, indicating that stage 2 alone brings little change. With verification, the same model achieves higher quality than both settings while using fewer backbone executions than $\ell=0$, so the quality gain comes from accepting predictions selectively rather than from additional computation. The speedup is smaller than the fraction of skipped windows would suggest, since each committed chunk still requires a cache-update pass that cannot be skipped. Please refer to Appendix~\ref{app:efficiency} for more details.

\begin{wraptable}{r}{0.45\textwidth}
\vspace{-18pt}
\centering
\caption{Ablation on prediction horizons.}
\vspace{0.1cm}
\label{tab:abl_horizon}
\small
\setlength{\tabcolsep}{6pt}
\begin{tabular}{lcc}
\toprule
Horizons & Quality$\uparrow$ & Drift$\downarrow$ \\
\midrule
$\{1\}$ & 81.00 & 4.08 \\
$\{1,2\}$ (Ours) & \textbf{81.29} & 3.29 \\
$\{1,2,4\}$ & 80.21 & \textbf{2.89} \\
\bottomrule
\end{tabular}
\vspace{-8pt}
\end{wraptable}

\textbf{Prediction horizon.}
Table~\ref{tab:abl_horizon} varies the set of future offsets predicted by the heads. For example, $\{1, 2\}$ means that $G_1$ and $G_2$ predict the first and second chunks after the window. All models are initialized with causal consistency distillation, trained with stage 1 only, and evaluated with standard inference, so the heads do not participate in generation and any difference comes from the learned representation. Compared with predicting only the next chunk, adding a second head improves both quality and drift. This suggests that predicting further ahead strengthens the speculative representation and mitigates temporal shortsightedness. A longer horizon is not always better, however. Extending the prediction to the fourth chunk yields the lowest drift but degrades quality, likely because distant chunks are harder to predict and provide noisier supervision. We therefore use $K=2$ heads with offsets $\{1, 2\}$, which gives the best quality with low drift and also supports the accelerated inference with $\ell \le 2$ in Table~\ref{tab:abl_inference}.

\begin{table}[t]
\centering
\vspace{-0.5cm}
\caption{\textbf{Ablation on inference mode.} Results of the model trained with stage 1 only and with both stages, under different inference modes.}
\vspace{0.1cm}
\label{tab:abl_inference}
\small
\setlength{\tabcolsep}{5pt}
\begin{tabular}{llccc}
\toprule
Train & Inference & FPS$\uparrow$ & Quality$\uparrow$ & Drift$\downarrow$ \\
\midrule
\multirow{3}{*}{Stage 1} & $\ell=0$ & 7.69 & 81.29 & \textbf{3.29} \\
 & Unverified, $\ell=1$ & 12.03 & 81.80 & 4.22 \\
 & Unverified, $\ell=2$ & \textbf{14.51} & 81.16 & 4.51 \\
\midrule
 & $\ell=0$ & 7.81 & 81.63 & 3.42 \\
{Stage 1 + 2} & Unverified, $\ell=1$ & 11.95 & 81.98 & 4.14 \\
 & Verified, $\ell=1$ (Ours) & 9.23 & \textbf{82.71} & 3.64 \\
\bottomrule
\end{tabular}
\vspace{-0.1cm}
\end{table}

\begin{table}[t]
\centering
\caption{\textbf{Ablation on design choices under stage $1$ and standard inference.} Each variant changes one component of the ODE-initialized model (row ``ODE'').}
\vspace{0.1cm}
\label{tab:abl_structure}
\small
\setlength{\tabcolsep}{5pt}
\begin{tabular}{llccccc}
\toprule
Factor & Setting & Quality$\uparrow$ & Subject$\uparrow$ & Dynamic$\uparrow$ & Imaging$\uparrow$ & Drift$\downarrow$ \\
\midrule
\multirow{2}{*}{Architecture}
 & Shared recurrent head & 77.14 & 84.72 & \textbf{100.0} & 63.47 & 5.10 \\
 & Sequential heads & 78.37 & 85.30 & \textbf{100.0} & 65.01 & 4.16 \\
\midrule
\multirow{2}{*}{Head input} & 3 chunks & 79.58 & 85.48 & \textbf{100.0} & 65.74 & 4.04 \\
 & 5 chunks & 77.25 & 84.41 & \textbf{100.0} & 64.13 & 4.61 \\
\midrule
\multirow{2}{*}{Head output} & 3 chunks & 80.43 & 94.23 & 17.0 & \textbf{71.40} & 2.58 \\
 & 5 chunks & 80.87 & \textbf{95.28} & 23.5 & 70.88 & \textbf{2.05} \\
\midrule
\multirow{2}{*}{Loss} & w/o $\mathcal{L}_{\mathrm{feat}}, \mathcal{L}_{\mathrm{vel}}$ & 80.05 & 84.36 & 98.5 & 68.24 & 5.32 \\
 & w/o $\mathcal{L}_{\mathrm{feat}}$ & 77.94 & 85.17 & \textbf{100.0} & 64.30 & 4.55 \\
\midrule
\multirow{2}{*}{Initialization} & ODE & 80.64 & 87.01 & \textbf{100.0} & 65.99 & 3.28 \\
 & Causal CD (Ours) & \textbf{81.29} & 88.05 & \textbf{100.0} & 68.12 & 3.29 \\
\bottomrule
\end{tabular}
\vspace{-0.3cm}
\end{table}

\textbf{Design choices.}
Table~\ref{tab:abl_structure} ablates our design choices under the same stage-1, standard-inference setting, where each variant changes one component of the ODE-initialized model. We report the VBench Quality score together with three of its dimensions (subject consistency, dynamic degree, and imaging quality), and the MUSIQ drift. We first examine the architecture of the prediction heads. Sharing one head recursively across offsets, or using separate heads that each take the previous prediction as input, both underperform our parallel heads. We then vary what the heads read and predict. Feeding them more past chunks makes the prediction harder to learn and lowers quality, whereas supervising more positions of the future window raises the scores but collapses the dynamic degree and produces nearly static videos. Mapping the current chunk to a single future chunk therefore works best. 
For the loss, training the same ODE-initialized model with DMD only, which is equivalent to Rolling Forcing, yields lower quality and substantially higher drift, showing that the gain comes from future alignment rather than initialization.
Removing the feature alignment loss degrades performance even below this DMD-only model, which indicates that learning the representation itself is essential. Finally, causal consistency distillation~\citep{zhao2026causal} further improves over ODE~\citep{yin2025causvid}, showing that our method is complementary to a stronger initialization.

\vspace{-0.3cm}

\section{Conclusion}
\vspace{-0.3cm}

We introduced Prospective Forcing to address temporal shortsightedness in autoregressive long video diffusion. Rather than supervising only the next chunk, it trains lightweight heads to anticipate chunks at multiple future offsets, distilling targets from the backbone's own later predictions. This turns the shared history into a speculative representation that looks beyond the immediate future. The same heads further enable verifier-guided speculative inference, which skips backbone executions whenever a predicted chunk is accepted and otherwise falls back to standard generation. Prospective Forcing outperforms baselines across multiple quantitative metrics and is preferred in human evaluations, while also achieving faster sampling. \textbf{Limitations.} Although Prospective Forcing accelerates sampling, the inference process still requires additional cache-update passes, which limits the speedup. We leave reducing this overhead to future work.

\newpage

\section*{AI Use Statement}
We used generative AI tools to assist with editing the manuscript for grammar and readability, with writing and modifying parts of the experimental code, and with parts of the theoretical derivations. All AI-assisted text, code, and derivations were carefully reviewed, tested, and verified by the authors, who take full responsibility for the content of this paper.

\section*{Reproducibility Statement}
Anonymized source code is provided in the supplementary material.
Implementation details and evaluation protocols are
given in Appendices~\ref{app:implementation} and \ref{app:metrics}, respectively.

\section*{Ethics Statement}
Our user study involved voluntary participants who compared generated videos. No personal data was collected during the user study.

\bibliography{refs}
\bibliographystyle{iclr2027_conference}

\appendix
\newpage

\section{Proof of Theorem~\ref{theorem:block-identifiability}}
\label{app:proof}

We first show how future alignment yields the conditional independence required of the estimated model, and then prove the case $L=2$, for which the recovery window is $\mathbf x_{t:t-2}$ and the predictive window is $\mathbf x_{t+4:t+2}$. The extension to arbitrary finite $L\ge1$ is given at the end of this section.

\begin{lemma}[Future alignment induces a predictive bottleneck]
\label{a:bottleneck}
Consider the estimated model with $W=1$, in which the noise of each chunk is drawn independently of the past, each chunk is a measurable function of its backbone velocity and its noise (through the denoising update and the VAE decoder), and the prediction heads map $\widehat{\mathbf s}_t$ and this noise to the features and velocities of the chunks in $\mathbf x_{t+2L:t+L}$. If $S_0(\vartheta)=1$, then $\mathbf x_{t+2L:t+L}\mathrel{\perp\!\!\!\perp}\mathbf x_{t:t-L}\mid\widehat{\mathbf s}_t$ in the estimated model.
\end{lemma}

\begin{proof}
$S_0(\vartheta)=1$ means that the population alignment losses vanish, so almost surely the head outputs coincide with the backbone features and velocities computed when the corresponding chunks enter the window under the same noise realization. Hence $\mathbf x_{t+2L:t+L}=g(\widehat{\mathbf s}_t,\boldsymbol\epsilon_{t+L:t+2L})$ almost surely for a measurable map $g$. The noise $\boldsymbol\epsilon_{t+L:t+2L}$ is independent of $\mathbf x_{t:t-L}$, and $\widehat{\mathbf s}_t=\widehat{\mathbf m}(\mathbf x_{t:t-L})$ is a function of the recovery window, so the conditional law of $\mathbf x_{t+2L:t+L}$ given $\mathbf x_{t:t-L}$ depends on it only through $\widehat{\mathbf s}_t$.
\end{proof}

\begin{proof}[Proof of Theorem~\ref{theorem:block-identifiability}]
Fix $t$. The latent variables entering $\mathbf{x}_{t:t-2}$ belong to
$\mathbf{s}_{t:t-4}$, whereas those entering $\mathbf{x}_{t+4:t+2}$ belong to $\mathbf{s}_{t+4:t}$. By first-order Markovity of $\{\mathbf{s}_t\}$, the latent blocks $\mathbf{s}_{t-1:t-4}$ and $\mathbf{s}_{t+4:t+1}$ are conditionally independent given $\mathbf{s}_t$. Together with the conditional independence of the chunks given the states, this gives
\begin{equation}
 \mathbf{x}_{t+4:t+2}\mathrel{\perp\!\!\!\perp}\mathbf{x}_{t:t-2}|\mathbf{s}_t.
 \label{a:cut}
\end{equation}
The future window starts at $t+2$ because $\mathbf{x}_{t+1}$ may still depend on $\mathbf{s}_{t-1}$.
In the estimated model, Lemma~\ref{a:bottleneck} gives
$\mathbf{x}_{t+4:t+2}\mathrel{\perp\!\!\!\perp}\mathbf{x}_{t:t-2}|\widehat{\mathbf{s}}_t$; this is the only step that uses the selection $S_0(\vartheta)=1$.

Enumerate the lower rectangles in $\mathbb R^{3d_x}$ with rational endpoints,
intersected with $\mathcal Y$, as $(A_j)_{j\ge1}$. Their probabilities
determine a distribution on $\mathcal Y$. Define
\begin{equation}
 \Phi(\mathbf{s})=\left(\int_{A_j}p(\mathbf y|\mathbf{s})\,d\mathbf y\right)_{j\ge1},
 \qquad
 \widehat\Phi(\widehat{\mathbf{s}})=
 \left(\int_{A_j}\widehat p(\mathbf y|\widehat{\mathbf{s}})\,d\mathbf y\right)_{j\ge1}.
 \label{a:prediction-sequences}
\end{equation}
These are Borel maps into $[0,1]^{\mathbb N}$. They are independent of
$t$ by time homogeneity, and injective by condition~\ref{a:predictive-separation} and its counterpart for the estimated model. They record the complete predictive
distributions; no differentiability of $\Phi$ or $\widehat\Phi$ is used.

For each $j$, the tower property, Equ.~(\ref{a:cut}), and exact recovery give
\begin{equation}
 \begin{aligned}
 \mathbb{P}(\mathbf{x}_{t+4:t+2}\in A_j|\mathbf{x}_{t:t-2})
 &=\mathbb E\left[
   \mathbb{P}(\mathbf{x}_{t+4:t+2}\in A_j|\mathbf{s}_t,\mathbf{x}_{t:t-2})
   \,\middle|\,\mathbf{x}_{t:t-2}\right]\\
 &=\mathbb E[\Phi_j(\mathbf{s}_t)|\mathbf{x}_{t:t-2}]\\
 &=\Phi_j(\mathbf m(\mathbf{x}_{t:t-2}))
 \quad\text{almost surely}.
 \end{aligned}
 \label{a:conditional-prediction}
\end{equation}
The estimated model gives the corresponding expression
$\widehat\Phi_j(\widehat{\mathbf m}(\mathbf{x}_{t:t-2}))$.
Observational equivalence makes the left-hand sides equal under the
common observation law. Countability of the event collection gives
\begin{equation}
 \Phi(\mathbf m(\mathbf{x}_{t:t-2}))
 =\widehat\Phi(\widehat{\mathbf m}(\mathbf{x}_{t:t-2}))
 \quad\text{almost surely}.
 \label{a:prediction-equality}
\end{equation}
Equivalently, for almost every recovery window, the predictive densities satisfy
\begin{equation}
 p(\mathbf y|\mathbf m(\mathbf{x}_{t:t-2}))
 =\widehat p(\mathbf y|\widehat{\mathbf m}(\mathbf{x}_{t:t-2}))
 \quad\text{for Lebesgue-almost every }\mathbf y\in\mathcal Y.
 \label{a:predictive-density-equality}
\end{equation}

By the Lusin--Souslin theorem, an injective Borel map between standard Borel spaces has a Borel image and a Borel inverse on that image~\citep{kechris2012classical}. Hence $\mathcal{P}=\Phi(\mathcal{S})\cap\widehat\Phi(\widehat{\mathcal{S}})$, $\mathcal{S}_0=\Phi^{-1}(\mathcal{P})$ and $\widehat{\mathcal{S}}_0=\widehat\Phi^{-1}(\mathcal{P})$ are Borel sets. Equ.~(\ref{a:prediction-equality}) shows that $\mathcal{S}_0$ and $\widehat{\mathcal{S}}_0$ have probability one for the respective states at every time. On these sets, the maps $\varphi_0=\Phi^{-1}\circ\widehat\Phi:\widehat{\mathcal{S}}_0\to\mathcal{S}_0$ and $\varphi_0^{-1}=\widehat\Phi^{-1}\circ\Phi$ are mutually inverse Borel maps. In particular, $\mathbf m(\mathbf{x}_{t:t-2})=\varphi_0(\widehat{\mathbf m}(\mathbf{x}_{t:t-2}))$ almost surely. Equality of the full prediction ranges is not needed.

We next obtain a smooth correspondence using the recovery maps. Extend
$\varphi_0$ and $\varphi_0^{-1}$ arbitrarily as Borel maps to the
remaining states. These extensions need not be inverses there.
At time $t_0$ in condition~\ref{a:regularity}, positivity of the window density makes the almost-sure relation above valid for Lebesgue-almost every $\mathbf u\in\mathcal U$:
\begin{equation}
 \mathbf m(\mathbf u)=\varphi_0(\widehat{\mathbf m}(\mathbf u)).
 \label{a:measurable-recovery-factorization}
\end{equation}

By the counterpart of condition~\ref{a:recovery} for the estimated model, $\widehat{\mathbf m}$ is $C^3$ with full row rank. The local submersion theorem for $C^3$ maps gives, near every window, local $C^3$ coordinates $\kappa:O_1\times O_2\to\mathcal U$, with $O_1\subseteq\widehat{\mathcal{S}}$ and $O_2\subseteq\mathbb R^{3d_x-n}$ open, such that $\widehat{\mathbf m}(\kappa(a,b))=a$, see~\citep{abraham2012manifolds}. The coordinate change preserves null sets, so Equ.~(\ref{a:measurable-recovery-factorization}) becomes $\mathbf m(\kappa(a,b))=\varphi_0(a)$ for Lebesgue-almost every $(a,b)$.
Fubini's theorem gives this equality for almost every $b$ at almost every $a$. For each such $a$, continuity in $b$ extends it to all $b$.
Fix $b_0\in O_2$ and define $\varphi_{O_1}(a)=\mathbf m(\kappa(a,b_0))$. This map is $C^3$ and satisfies $\mathbf m(\kappa(a,b))=\varphi_{O_1}(a)$ for almost every $a$ and all $b$. Continuity in $a$ then extends this identity to all $(a,b)\in O_1\times O_2$.

On overlapping state neighborhoods, these local maps agree almost everywhere with $\varphi_0$, so continuity makes them agree at every point. Surjectivity of $\widehat{\mathbf m}$ ensures that the neighborhoods cover $\widehat{\mathcal{S}}$. Thus they define a global $C^3$ map $\varphi$ satisfying $\mathbf m(\mathbf u)=\varphi(\widehat{\mathbf m}(\mathbf u))$ for every $\mathbf u\in\mathcal U$.
Applying the same argument to the reverse almost-sure relation gives a $C^3$ map $\boldsymbol\rho:\mathcal{S}\to\widehat{\mathcal{S}}$ such that $\widehat{\mathbf m}=\boldsymbol\rho\circ\mathbf m$ on $\mathcal U$.
Surjectivity of both recovery maps implies $\boldsymbol\rho\circ\varphi=\operatorname{id}_{\widehat{\mathcal{S}}}$ and $\varphi\circ\boldsymbol\rho=\operatorname{id}_{\mathcal{S}}$.
Consequently, $\varphi$ is a $C^3$ diffeomorphism. It is unique because $\widehat{\mathbf m}$ is onto. The recovery maps and the identity on $\mathcal U$ are independent of time and history.
Since every recovery window belongs to $\mathcal U$ almost surely and $\mathbb Z$ is countable, the recovered states satisfy the same identity simultaneously at all times on an event of probability one.
\end{proof}

\begin{lemma}[Global disentanglement on product domains]
\label{a:global}
Let $\varphi:\widehat{\mathcal{S}}\to\mathcal{S}$ be a $C^3$ diffeomorphism between the products of open intervals $\widehat{\mathcal S}=\prod_{i=1}^n\widehat I_i$ and $\mathcal S=\prod_{k=1}^n I_k$. If its Jacobian has at most one nonzero entry in each row at every point, then $\varphi$ and $\varphi^{-1}$ are fixed permutations composed with fixed one-dimensional $C^3$ diffeomorphisms.
\end{lemma}

\begin{proof}
Write $\varphi=(\varphi_1,\ldots,\varphi_n)$. If the Jacobian of $\varphi$ has at most one nonzero entry in each row at every point, invertibility implies exactly one nonzero entry in each row and each column. For a row $k$, its nonzero column is locally constant: a nonzero entry stays nonzero in a sufficiently small neighborhood by continuity, and no other entry in that row can become nonzero there. Since $\widehat{\mathcal{S}}$ is connected, this column is one fixed index $\sigma(k)$ throughout the domain. The map $\sigma$ is a permutation.

For $j\ne\sigma(k)$, $\partial\varphi_k/\partial\hat s_j=0$ on $\widehat{\mathcal{S}}$.
Any two points with the same $\hat s_{\sigma(k)}$ can be joined by a sequence of coordinate segments in the product domain, along each of which $\varphi_k$ is constant. Hence $\varphi_k(\widehat{\mathbf{s}})=q_k(\hat s_{\sigma(k)})$.
Each $q_k$ is $C^3$ and its derivative is nonzero on
$\widehat I_{\sigma(k)}$, so it is strictly monotone there.
Surjectivity of $\varphi$ implies
that $q_k$ maps onto $I_k$. Its inverse is therefore a global $C^3$
diffeomorphism. Setting $\pi(i)=\sigma^{-1}(i)$ and
$\mathcal T_i=q_{\pi(i)}^{-1}$ gives the claimed form of $\varphi^{-1}$.
\end{proof}

\begin{theorem}[Component-wise identifiability, $L=2$]
\label{theorem:comp-identifiability}
Consider the setting of Theorem~\ref{theorem:block-identifiability} with $L=2$, so that $\mathcal U,\mathcal Y\subseteq\mathbb R^{3d_x}$ are the open domains for the recovery window $\mathbf x_{t:t-2}$ and the predictive window $\mathbf x_{t+4:t+2}$, respectively. Suppose the true model and its transition density satisfy the following assumptions:
\begin{enumerate}
\item (Smooth and positive density)
All state marginal densities and the transition density are strictly positive on their respective domains. Each coordinate log density $\log p(s_{k,t}|\mathbf{s}_{t-1})$ is $C^3$ (i.e., has continuous third-order partial derivatives).
For some $t_0$, $\mathbf x_{t_0:t_0-2}$ has a Lebesgue density positive almost everywhere on $\mathcal U$.
For every $\mathbf{s}\in\mathcal{S}$, the conditional distribution of $\mathbf x_{t+4:t+2}$ given $\mathbf{s}_t=\mathbf{s}$ admits a jointly Borel Lebesgue density.

\item (Conditional independence)
The current-state coordinates are mutually independent conditional on the previous state: $p(\mathbf{s}_t|\mathbf{s}_{t-1}) =\prod_{k=1}^{n}p(s_{k,t}|\mathbf{s}_{t-1})$.
\label{a:ci}

\item (Smooth finite-context recovery)
There is a surjective $C^3$ map $\mathbf m:\mathcal U\to\mathcal S$ such that $\mathbf{s}_t=\mathbf m(\mathbf{x}_{t:t-2})$ almost surely for every $t$.
Its Jacobian has full row rank: $\operatorname{rank}D\mathbf m(\mathbf u)=n$ for every $\mathbf u\in\mathcal U$.

\item (Predictive density separation)
For any distinct states $\mathbf{s},\mathbf{s}'\in\mathcal{S}$, the set $\{\mathbf y\in\mathcal Y: p(\mathbf y|\mathbf{s})\ne p(\mathbf y|\mathbf{s}')\}$ has positive probability under the marginal distribution of $\mathbf x_{t+4:t+2}$.

\item (Sufficient variability)
Let $\eta_{k,t}=\log p(s_{k,t}|\mathbf{s}_{t-1})$ and, for $l=1,\ldots,n$, define
\begin{equation}
 \boldsymbol\omega_{l,t}(\mathbf{s}_t,\mathbf{s}_{t-1})=\begin{pmatrix}
  \dfrac{\partial^2\eta_{1,t}}{\partial s_{1,t}\partial s_{l,t-1}},\cdots,
  \dfrac{\partial^2\eta_{n,t}}{\partial s_{n,t}\partial s_{l,t-1}},
  \dfrac{\partial^3\eta_{1,t}}{\partial s_{1,t}^2\partial s_{l,t-1}},\cdots,
  \dfrac{\partial^3\eta_{n,t}}{\partial s_{n,t}^2\partial s_{l,t-1}}
 \end{pmatrix}^{\top}\in\mathbb{R}^{2n}.
\end{equation}
For every fixed $\mathbf{s}_t\in\mathcal{S}$, there exist $2n$ pairs $(l_u,\mathbf{s}_{t-1}^{(u)})$, with $l_u\in\{1,\ldots,n\}$ and $\mathbf{s}_{t-1}^{(u)}\in\mathcal{S}$, such that the vectors $\boldsymbol\omega_{l_u,t}(\mathbf{s}_t,\mathbf{s}_{t-1}^{(u)})$ are linearly independent.
\label{a:variability}
\end{enumerate}
Then, for any estimated model that satisfies these assumptions with $(\widehat{\mathbf m},\widehat p,\widehat{\mathcal S})$ in place of $(\mathbf m,p,\mathcal S)$, meets the requirements on the estimated model in Theorem~\ref{theorem:block-identifiability}, and satisfies $S_0(\vartheta)=1$, there exist a single permutation $\pi$ and fixed $C^3$ diffeomorphisms $\mathcal T_i:I_{\pi(i)}\to\widehat I_i$ for which $\widehat m_i(\mathbf{x}_{t:t-2}) =\mathcal T_i\bigl(m_{\pi(i)}(\mathbf{x}_{t:t-2})\bigr)$ holds for $i=1,\ldots,n$.
\end{theorem}

\begin{proof}
By Theorem~\ref{theorem:block-identifiability}, there is a fixed $C^3$ diffeomorphism
$\varphi:\widehat{\mathcal{S}}\to\mathcal{S}$. On the common observed-process space, define
$\mathbf{s}_t=\mathbf m(\mathbf{x}_{t:t-2})$ and
$\widehat{\mathbf{s}}_t=\widehat{\mathbf m}(\mathbf{x}_{t:t-2})$.
Assumption~\ref{a:recovery} ensures that these sequences have the latent-process distributions of their respective models, and Theorem~\ref{theorem:block-identifiability} gives
$\mathbf{s}_t=\varphi(\widehat{\mathbf{s}}_t)$ simultaneously for every $t$ almost surely.
In particular, the same function also relates the previous states: $\widehat{\mathbf{s}}_{t-1}=\varphi^{-1}(\mathbf{s}_{t-1})$.

Write $\varphi=(\varphi_1,\ldots,\varphi_n)$ and denote the Jacobian by
\begin{equation}
 \mathbf H_t=\frac{\partial\varphi(\widehat{\mathbf{s}}_t)}{\partial\widehat{\mathbf{s}}_t},
 \qquad
 H_{ki,t}=\frac{\partial \varphi_k(\widehat{\mathbf{s}}_t)}{\partial\hat s_{i,t}}.
 \label{a:jacobian}
\end{equation}
The subscript $t$ records the state at which the same fixed Jacobian
function is evaluated; it does not introduce a time-dependent
transformation. In the density identities below, $\widehat{\mathbf{s}}_t$ and
$\mathbf{s}_{t-1}$ denote freely varying state values. Conditional change of
variables gives
\begin{equation}
 \widehat p(\widehat{\mathbf{s}}_t|\varphi^{-1}(\mathbf{s}_{t-1}))
 =p(\varphi(\widehat{\mathbf{s}}_t)|\mathbf{s}_{t-1})\,|\det\mathbf H_t|.
 \label{a:change-of-variables}
\end{equation}
This identity first holds for almost every $(\widehat{\mathbf{s}}_t,\mathbf{s}_{t-1})$ with respect to Lebesgue measure, since the previous-state density is strictly positive.
The same condition and the smoothness of $\varphi$ ensure continuity of both sides, which extends the identity to all
$(\widehat{\mathbf{s}}_t,\mathbf{s}_{t-1})\in\widehat{\mathcal{S}}\times\mathcal{S}$.
Recall $\eta_{k,t}=\log p(s_{k,t}|\mathbf{s}_{t-1})$. Evaluating these functions at $\mathbf{s}_t=\varphi(\widehat{\mathbf{s}}_t)$, conditional independence Assumption~\ref{a:ci} then gives
\begin{equation}
 \log\widehat p(\widehat{\mathbf{s}}_t|\varphi^{-1}(\mathbf{s}_{t-1}))
 =\sum_{k=1}^{n}\eta_{k,t}+\log|\det\mathbf H_t|.
 \label{a:log-density}
\end{equation}

The estimated model also satisfies the assumptions of Theorem~\ref{theorem:comp-identifiability}, so its transition density satisfies condition~\ref{a:ci}. For $i\ne j$, with
$\mathbf{s}_{t-1}$ held fixed,
\begin{equation}
 \frac{\partial^2
   \log\widehat p(\widehat{\mathbf{s}}_t|\varphi^{-1}(\mathbf{s}_{t-1}))}
  {\partial\hat s_{i,t}\partial\hat s_{j,t}}=0.
 \label{a:zero-mixed}
\end{equation}
Using the chain rule in Equ.~(\ref{a:log-density}), we obtain
\begin{equation}
 0=\sum_{k=1}^{n}\left[
   \frac{\partial^2\eta_{k,t}}{\partial s_{k,t}^2}H_{ki,t}H_{kj,t}
   +\frac{\partial\eta_{k,t}}{\partial s_{k,t}}
    \frac{\partial H_{ki,t}}{\partial\hat s_{j,t}}\right]
   +\frac{\partial^2\log|\det\mathbf H_t|}
    {\partial\hat s_{i,t}\partial\hat s_{j,t}},
 \label{a:mixed-current}
\end{equation}
where all derivatives of $\eta_{k,t}$ are evaluated at $(\varphi_k(\widehat{\mathbf{s}}_t),\mathbf{s}_{t-1})$.

Differentiate Equ.~(\ref{a:mixed-current}) with respect to $s_{l,t-1}$, holding $\widehat{\mathbf{s}}_t$ fixed. Since $\varphi$ and $\mathbf H_t$ depend only on $\widehat{\mathbf{s}}_t$, the determinant term vanishes and no derivative of either function with respect to the previous state appears:
\begin{equation}
 0=\sum_{k=1}^{n}\left[
  \frac{\partial^3\eta_{k,t}}{\partial s_{k,t}^2\partial s_{l,t-1}}
   H_{ki,t}H_{kj,t}
  +\frac{\partial^2\eta_{k,t}}{\partial s_{k,t}\partial s_{l,t-1}}
   \frac{\partial H_{ki,t}}{\partial\hat s_{j,t}}
  \right].
 \label{a:history-derivative}
\end{equation}
Fix $\widehat{\mathbf{s}}_t$, and hence $\mathbf{s}_t=\varphi(\widehat{\mathbf{s}}_t)$. The $2n$ coefficients
\[
 \left(
 \frac{\partial H_{1i,t}}{\partial\hat s_{j,t}},\ldots,
 \frac{\partial H_{ni,t}}{\partial\hat s_{j,t}},
 H_{1i,t}H_{1j,t},\ldots,H_{ni,t}H_{nj,t}
 \right)^{\mathsf T}
\]
do not change as $\mathbf{s}_{t-1}$ varies.
Equ.~(\ref{a:history-derivative}) says that their inner product with $\boldsymbol\omega_{l,t}(\mathbf{s}_t,\mathbf{s}_{t-1})$ is zero for every $l$ and every previous state. The $2n$ linearly independent vectors in sufficient variability Assumption~\ref{a:variability} therefore force all these coefficients to vanish. In particular,
\begin{equation}
 H_{ki,t}H_{kj,t}=0
 \quad\text{for all }k,\ i\ne j,\ \widehat{\mathbf{s}}_t\in\widehat{\mathcal{S}}.
 \label{a:sparsity}
\end{equation}

Apply Lemma~\ref{a:global} to $\varphi$. There are fixed one-dimensional $C^3$ diffeomorphisms $\mathcal T_i$ and a fixed permutation $\pi$ such that $\hat s_{i,t}=\mathcal T_i(s_{\pi(i),t})$ almost surely for every $t$.
Substitution of the recovery maps proves that $\widehat m_i(\mathbf{x}_{t:t-2}) =\mathcal T_i\bigl(m_{\pi(i)}(\mathbf{x}_{t:t-2})\bigr)$ holds for $i=1,\ldots,n$.
\end{proof}

\paragraph{Extension to arbitrary finite $L$.}
Fix $L\ge1$. Replace $\mathbf x_{t:t-2}$ by $\mathbf x_{t:t-L}$ and $\mathbf x_{t+4:t+2}$ by $\mathbf x_{t+2L:t+L}$ throughout the alignment argument. Because each chunk $\mathbf x_{t'}$ depends only on $\mathbf{s}_{t':t'-L}$, the latent variables entering the recovery window belong to $\mathbf{s}_{t:t-2L}$, while those entering the predictive window belong to $\mathbf{s}_{t+2L:t}$. Thus the two latent blocks share only $\mathbf{s}_t$; first-order Markovity and conditional independence of the chunks give $\mathbf x_{t+2L:t+L}\mathrel{\perp\!\!\!\perp}\mathbf x_{t:t-L}\mid\mathbf{s}_t$, and Lemma~\ref{a:bottleneck} gives the same for the estimated model. The predictive-separation and smooth-recovery steps of Theorem~\ref{theorem:block-identifiability} then apply with $\mathbb R^{(L+1)d_x}$ in place of $\mathbb R^{3d_x}$ and $\mathbb R^{(L+1)d_x-n}$ in place of $\mathbb R^{3d_x-n}$, yielding the same fixed $C^3$ diffeomorphism $\varphi$. Lemma~\ref{a:global} and the change-of-variables and sufficient-variability argument above involve only the one-step transition density $p(\mathbf{s}_t|\mathbf{s}_{t-1})$ and are unchanged. Consequently, the stated component-wise identification holds for every fixed finite $L\ge1$.

\section{Related Work}
\label{app:related_work}

\paragraph{Autoregressive video generation.}
Modern video diffusion models typically denoise all frames in a fixed temporal window with bidirectional attention~\citep{blattmann2023stablevideodiffusion,yang2024cogvideox,kong2024hunyuanvideo,wan2025wan}. This design supports high-quality offline synthesis but requires the complete clip to be generated jointly. Autoregressive video models instead generate frames or chunks sequentially. This causal formulation reduces viewing latency and enables efficient streaming over extended horizons through KV caching~\citep{yin2025causvid,huang2025selfforcingbridgingtraintest}. Within this family, Diffusion Forcing~\citep{chen2024diffusionforcing} assigns an independent noise level to each token. It connects causal next-token prediction with full-sequence diffusion and supports flexible generation horizons. Another line uses distribution matching distillation (DMD)~\citep{yin2024dmd} to reduce the sampling cost of causal generators. CausVid~\citep{yin2025causvid} combines ODE initialization with asymmetric DMD to obtain a few-step causal student for real-time generation. Self Forcing~\citep{huang2025selfforcingbridgingtraintest} trains on self-generated rollouts and applies a holistic distribution-matching objective to mitigate exposure bias. Causal Forcing~\citep{zhu2026causalforcing} uses an autoregressive ODE teacher to bridge the architectural gap between bidirectional teachers and causal students. Causal Forcing++~\citep{zhao2026causal} replaces costly offline ODE trajectories with causal consistency distillation, making few-step initialization scalable. Beyond DMD-based post-training, Resampling Forcing~\citep{guo2025resamplingforcing} uses self-resampled histories and a sparse causal mask for teacher-free end-to-end training.

\paragraph{Long video generation.}
Consistency over long autoregressive rollouts has been addressed through training, denoising, and context-management strategies. Self-Forcing++~\citep{cui2026selfforcingplusplus} mitigates quality degradation when the student extrapolates beyond its training horizon by providing teacher guidance on segments sampled from self-generated long videos. Rolling Forcing~\citep{liu2025rolling} mitigates error accumulation by jointly denoising frames at progressively increasing noise levels within a rolling window and retaining initial-frame KV states as a global context anchor. StreamingT2V~\citep{DBLP:conf/cvpr/HenschelKPHTWNS25} uses features from the preceding chunk to facilitate smooth transitions and preserves scene and object features from the first video chunk to prevent forgetting of the initial scene. LongLive~\citep{yang2026longlive} uses streaming long tuning to align long-video training and inference, together with short window attention and a frame sink to improve efficiency while preserving long-range consistency. Context Forcing~\citep{chen2026contextforcing} resolves the mismatch between a long-context student and a short-context teacher by training with a long-context teacher, while its Slow-Fast Memory architecture makes the growing context computationally tractable. DynaMem~\citep{lin2026dynamem} addresses semantic drift, motion decay, and appearance instability through Semantic-Adaptive Hierarchical Memory, Motion-Prioritized Optimization, and Reference-Anchored Perceptual Alignment. Recent training-free approaches such as Deep Forcing and LoL implement context management at inference time by modifying cache retention and/or positional encoding. Deep Forcing~\citep{yi2026deepforcing} introduces Deep Sink, which retains a substantially enlarged attention sink with temporal RoPE adjustment, and Participative Compression, which prunes redundant and degraded history from the KV cache. LoL~\citep{cui2026lol} introduces multi-head RoPE jitter to break inter-head attention homogenization and mitigate sink-collapse over extreme horizons. Instead, our method improves temporal consistency by learning representations with long-horizon predictive capabilities.

\paragraph{Future-predictive representations and multi-horizon generation.}
Predicting future observations has long served as an auxiliary supervision for learning sequential representations. PBL~\citep{pmlr-v119-guo20g} couples forward prediction of future-observation embeddings with reverse prediction of agent states from the corresponding latents, achieving a bootstrapping objective. SPR~\citep{schwarzer2021spr} instead rolls a learned transition model forward in latent space and matches its multi-step predictions to representations produced by an exponential-moving-average target encoder. In text generation, multi-token prediction~\citep{pmlr-v235-gloeckle24a} equips a shared language-model trunk with independent output heads for predicting multiple subsequent tokens, while Medusa~\citep{cai2024medusa} and EAGLE~\citep{li2024eagle} accelerate LLM decoding by drafting future tokens with lightweight heads and verifying them against the backbone. In image generation, Mirai~\citep{yu2026mirai} aligns internal representations with explicit or implicit foresight during autoregressive training while preserving the causal autoregressive inference path. The closest concurrent works in video generation are Video-Mirai and Next Forcing. Video-Mirai~\citep{yu2026videomirai} incorporates future information through alignment of JEPA-style representation~\citep{bardes2024vjepa}, while Next Forcing~\citep{xu2026nextforcing} uses recurrent future chunk prediction to improve performance in robotic tasks. In contrast, \method learns a shared representation for parallel multi-chunk prediction through independent heads and reuses these heads for verifier-guided speculative inference.

\section{Verifier Details}
\label{app:verifier}

\paragraph{Verifier prediction.}
The verifier $V_\phi$ is a lightweight MLP with trainable parameters $\phi$. It takes the current features $\mathbf{h}_W^{r}$ and the candidate features from~\eqref{eq:candidate_prediction} as inputs and outputs the selection and rejection probabilities. We use $\mathbf{h}_W^{r}$ and the candidates $\hat{\mathbf{h}}_{W+1}^{r,m}$ to decide which candidate to be selected:
\begin{equation}
\left(\pi_1^r,\ldots,\pi_M^r,\pi_{\mathrm{rej}}^r\right) = \operatorname{Softmax}\left(V_\phi\left(\operatorname{sg}\left[\mathbf{h}_W^{r}\right],\left\{\operatorname{sg}\left[\hat{\mathbf{h}}_{W+1}^{r,m}\right]\right\}_{m=1}^{M}\right)\right).
\label{eq:verifier_prediction}
\end{equation}
The stop-gradient operations prevent the verifier loss~\eqref{eq:verifier_loss} from updating the generator.

\paragraph{Label construction.}
After the window advances to $\mathcal{W}_{r+1}$, we compare each candidate with the backbone features $\mathbf{h}_W^{r+1}$ and define the normalized feature error of candidate $m$ as
\begin{equation}
E_m^r = \operatorname{MSE}\left(\operatorname{LN}\left(\hat{\mathbf{h}}_{W+1}^{r,m}\right),\operatorname{LN}\left(\mathbf{h}_W^{r+1}\right)\right).
\label{eq:candidate_error}
\end{equation}
Let $\delta$ denote the maximum acceptable error. We set $y_r$ to the candidate $m_r^\star=\arg\min_m E_m^r$ used in~\eqref{eq:candidate_alignment} if its error does not exceed $\delta$, and to the rejection label $\mathrm{rej}$ otherwise:
\begin{equation}
y_r =
\begin{cases}
m_r^\star, & E_{m_r^\star}^r \leq \delta,\\[4pt]
\mathrm{rej}, & \text{otherwise}.
\end{cases}
\label{eq:verifier_label}
\end{equation}

\section{Additional Experimental Details and Results}
\label{app:experiments}

\subsection{Implementation Details}
\label{app:implementation}
All models are trained on 8 NVIDIA A100 GPUs with a global batch size of eight and fully sharded data parallelism. We optimize the generator and the trainable score model with AdamW~\citep{loshchilov2019decoupled} using $(\beta_1,\beta_2)=(0, 0.999)$ and learning rates of $1.5\times10^{-6}$ and $4\times10^{-7}$, respectively, and update the score model five times per generator update. The generator uses an exponential moving average with decay 0.99 starting from iteration 200, and all reported results use the EMA weights. The VAE and text encoder remain frozen throughout training. Following \citet{liu2025rolling}, we use the denoising timesteps $\{1000, 800, 600, 400, 200\}$ with a timestep shift of 5, and the teacher uses a classifier-free guidance scale of 3. We train on the VidProM~\citep{wang2024vidprom} prompt set released by \citet{huang2025selfforcingbridgingtraintest}. Each chunk contains three latent frames, and each training rollout contains 21, 24, or 27 latent frames, i.e., 7 to 9 chunks.

In the first stage, we train for 3{,}000 iterations with the DMD objective and the feature and velocity alignment losses, without candidate embeddings or the verifier. In the second stage, we initialize from the stage-1 checkpoint, enable the candidate embeddings and the verifier, and continue training all parameters for 1{,}000 iterations. Each future prediction head is a lightweight transformer block with 12 attention heads and an MLP ratio of 2. It adds $\mathbf{h}_W^r$, $\mathbf{n}_{r,k}$ and a learned offset embedding, applies self-attention over the three latent frames of the chunk independently at each spatial location, and updates $\mathbf{h}_W^r$ through a noise-gated residual attention branch followed by an MLP. The heads do not attend to other window positions or to the text condition. The candidate embeddings are initialized from $\mathcal{N}(0, 0.02^2)$, with the first embedding set to zero so that the first candidate coincides with the original prediction. The verifier is a three-layer MLP with a hidden size of 96. The loss weights are $\lambda_{\mathrm{feat}}=0.1$, $\lambda_{\mathrm{vel}}=0.01$ and $\lambda_{\mathrm{aux}}=0.1$. For the candidate alignment loss we use $M=2$ candidates with $\beta=0.05$, and the verifier acceptance threshold is $\delta=0.12$.

Since our training setup follows Rolling Forcing, it is equivalent to our stage-1 model with ODE initialization and $\lambda_{\mathrm{feat}}=\lambda_{\mathrm{vel}}=0$, and serves as the DMD-only baseline in Table~\ref{tab:abl_structure}.

\subsection{Evaluation Metrics}
\label{app:metrics}
All metrics are computed on the same 200 videos per method.

\paragraph{VBench.} The Quality score follows the official aggregation~\citep{huang2024vbench}: each of the seven quality dimensions is normalized with the official minimum and maximum values, Dynamic Degree is weighted by 0.5 and the other six dimensions by 1, and the weighted sum is divided by 6.5 and multiplied by 100. None of the 200 prompts is applicable to the Color dimension, so we exclude Color and report the Semantic score as the average of the remaining eight normalized semantic dimensions. The Total score is $0.8\times\text{Quality} + 0.2\times\text{Semantic}$.

\paragraph{VBench-Long.} We evaluate six quality dimensions of VBench-Long~\citep{huang2025vbenchpp} in the \texttt{long\_custom\_input} mode: subject consistency, background consistency, motion smoothness, dynamic degree, aesthetic quality, and imaging quality. We aggregate them in the same way as the VBench Quality score, dividing the weighted sum by 5.5.

\paragraph{VideoScore.} We average the five dimensions of VideoScore v1.1~\citep{he2024videoscore} for each video and then average over videos.

\paragraph{MUSIQ and quality drift.} Following \citet{liu2025rolling}, we evaluate MUSIQ~\citep{ke2021musiq} on all frames of the first and last 5 seconds of each video. Let $H_i$ and $T_i$ be the mean MUSIQ scores of the first and last 5 seconds of video $i$. The MUSIQ avg. score averages $H_i$ and $T_i$ over all videos, and the quality drift is $\frac{1}{N}\sum_{i=1}^{N}|T_i - H_i|$. 

\subsection{Efficiency Measurement}
\label{app:efficiency}
We measure all methods on the same NVIDIA A100 80GB GPU with batch size one and bfloat16 precision. Table~\ref{tab:app_speed_breakdown} breaks down the per-video inference cost. Generating a 30-second video involves two kinds of backbone passes: sampling passes that denoise the rolling window, and cache-update passes that write the key--value states of each newly committed clean chunk into the cache. An accepted prediction advances the window by one extra chunk and skips one sampling pass. However, every committed chunk still requires its own cache-update pass, and the first windows cannot be skipped because the rolling window must first be filled. As a result, even when every prediction is accepted ($\ell=1$, unverified), the number of sampling passes only drops from 46 to 27, while the 46 cache-update passes, which account for about 20\% of the standard inference time, remain unchanged. This bounds the speedup of $\ell=1$ at $1.56\times$ rather than $2\times$. Our verifier accepts 32.1\% of the proposals (9 out of 28 rounds per video), which skips 9 sampling passes and yields a $1.20\times$ speedup. The overhead of the prediction heads and the verifier is negligible, below 0.2 seconds per video.

\begin{table}[h]
\centering
\caption{Per-video inference cost breakdown on a single A100 GPU (30-second video, 480 frames). }
\label{tab:app_speed_breakdown}
\small
\setlength{\tabcolsep}{4pt}
\begin{tabular}{lccccc}
\toprule
Inference & Accept rate & Sampling passes & Cache-update passes  & Speedup \\
\midrule
Standard ($\ell=0$) & -- & 46 & 46 & $1.00\times$ \\
Unverified, $\ell=1$ & 100\% & 27 & 46 & $1.56\times$ \\
Verified (Ours) & 32.1\% & 37 & 46 & $1.20\times$ \\
\bottomrule
\end{tabular}
\end{table}

\subsection{User Study}
\label{app:user_study}
We use randomly selected prompts from the evaluation set and generate 30-second videos at $832\times480$ and 16\,fps, compressed with H.264 at CRF 30 for web playback. Each participant completes five comparisons. In each comparison, the baseline is sampled uniformly from Rolling Forcing, LongLive, and Video-Mirai, our video is sampled uniformly from Ours and the stage-1 model with standard inference ($\ell=0$), and the left--right order is randomized. Participants choose the better video separately for visual quality, long-term consistency, and dynamics. We only count submitted responses that pass automatic and manual quality checks, which yields 59 valid participants and 295 comparisons, of which 130 involve Ours and 165 involve the stage-1 model. All preference rates of Ours in Figure~\ref{fig:user_study} are significantly above 50\% under a two-sided binomial test ($p<0.001$).

Table~\ref{tab:app_user_study_wo} reports the results of the stage-1 model. Even without accelerated inference, participants prefer it over all three baselines on every criterion. The preferences are significant ($p<0.05$) in all cases except for visual quality ($p=0.24$) and consistency ($p=0.56$) against LongLive.

\begin{table}[h]
\centering
\caption{User study of the stage-1 model with standard inference ($\ell=0$). Each entry is the percentage of pairwise comparisons in which participants preferred this model over the baseline. Average is computed over all comparisons of the three criteria.}
\label{tab:app_user_study_wo}
\small
\setlength{\tabcolsep}{5pt}
\begin{tabular}{lcccc}
\toprule
Stage 1 vs. & Visual Quality & Consistency & Dynamics & Avg. \\
\midrule
Rolling Forcing & 81.8\% & 86.4\% & 87.9\% & 85.4\% \\
LongLive & 59.6\% & 55.3\% & 68.1\% & 61.0\% \\
Video-Mirai & 78.8\% & 73.1\% & 86.5\% & 79.5\% \\
\bottomrule
\end{tabular}
\end{table}

\subsection{Per-Dimension Results}
\label{app:per_dimension}
Tables~\ref{tab:app_vbench_quality}--\ref{tab:app_musiq} report the individual dimensions behind the aggregate scores in Table~\ref{tab:main}. \textbf{Bold} and \underline{underline} denote the best and second-best results. Please note that some individual scores can not reflect the true visual quality. For example, in Table~\ref{tab:app_vbench_long}, the background consistency score is high, but the sampled videos have a quality shift and tend to stay static.

\begin{table}[h]
\centering
\caption{Per-dimension VBench quality scores (\%).}
\label{tab:app_vbench_quality}
\small
\setlength{\tabcolsep}{4pt}
\begin{tabular}{lccccccc}
\toprule
Method & Subject$\uparrow$ & Background$\uparrow$ & Flickering$\uparrow$ & Smoothness$\uparrow$ & Dynamic$\uparrow$ & Aesthetic$\uparrow$ & Imaging$\uparrow$ \\
\midrule
CausVid & 87.59 & 90.60 & \underline{98.13} & \textbf{99.06} & 21.00 & 57.35 & 67.47 \\
Causal Forcing++ & 85.33 & 90.48 & 95.93 & 97.71 & 72.00 & 60.33 & 66.13 \\
Self Forcing & \underline{89.59} & \underline{92.27} & \textbf{98.17} & \underline{98.77} & 24.50 & 61.07 & \textbf{69.57} \\
LongLive & \textbf{89.70} & \textbf{92.97} & 96.20 & 97.89 & 62.00 & \textbf{62.95} & \underline{69.34} \\
Rolling Forcing & 84.36 & 88.72 & 94.34 & 97.86 & \underline{98.50} & 58.19 & 68.24 \\
Video-Mirai & 82.85 & 88.02 & 95.99 & 96.92 & 47.00 & 57.30 & 64.44 \\
\midrule
\textbf{Ours} & 87.95 & 91.91 & 95.69 & 97.87 & \textbf{100.00} & \underline{61.92} & 68.86 \\
\bottomrule
\end{tabular}
\end{table}

\begin{table}[h]
\centering
\caption{Per-dimension VBench semantic scores (\%). The first row reports the number of prompts to which each dimension applies.}
\label{tab:app_vbench_semantic}
\small
\setlength{\tabcolsep}{4pt}
\resizebox{\linewidth}{!}{%
\begin{tabular}{lcccccccc}
\toprule
Method & Object$\uparrow$ & Multi-Obj.$\uparrow$ & Action$\uparrow$ & Spatial$\uparrow$ & Scene$\uparrow$ & App. Style$\uparrow$ & Temp. Style$\uparrow$ & Overall Cons.$\uparrow$ \\
\midrule
\textit{Valid videos} & 119 & 20 & 36 & 5 & 54 & 8 & 36 & 200 \\
\midrule
CausVid & 92.12 & \underline{77.50} & \textbf{94.44} & 77.28 & 35.07 & 21.29 & 22.19 & 23.51 \\
Causal Forcing++ & 95.06 & 76.88 & \textbf{94.44} & \underline{87.15} & 42.82 & \underline{21.35} & 22.72 & 24.46 \\
Self Forcing & \underline{95.48} & 65.31 & \underline{91.67} & \textbf{90.20} & 48.73 & 21.17 & 24.17 & 24.85 \\
LongLive & 95.43 & \textbf{80.94} & \textbf{94.44} & 80.74 & \underline{49.54} & 21.17 & \underline{24.95} & \underline{25.97} \\
Rolling Forcing & 87.61 & 53.44 & \underline{91.67} & 78.49 & \textbf{51.62} & 21.34 & 24.06 & 25.05 \\
Video-Mirai & 90.81 & 65.62 & \underline{91.67} & 81.84 & 45.49 & \textbf{21.71} & 23.53 & 25.00 \\
\midrule
\textbf{Ours} & \textbf{97.43} & 75.94 & \underline{91.67} & 74.61 & 46.76 & 20.13 & \textbf{25.56} & \textbf{26.23} \\
\bottomrule
\end{tabular}}
\end{table}

\begin{table}[h]
\centering
\caption{Per-dimension VBench-Long scores (\%).}
\label{tab:app_vbench_long}
\small
\setlength{\tabcolsep}{4pt}
\begin{tabular}{lcccccc}
\toprule
Method & Subject$\uparrow$ & Background$\uparrow$ & Smoothness$\uparrow$ & Dynamic$\uparrow$ & Aesthetic$\uparrow$ & Imaging$\uparrow$ \\
\midrule
CausVid & \underline{98.15} & \textbf{96.99} & \textbf{99.08} & 21.80 & 56.30 & 67.18 \\
Causal Forcing++ & 96.09 & 95.76 & 97.74 & 60.77 & 58.86 & 65.39 \\
Self Forcing & \textbf{98.19} & \underline{96.90} & \underline{98.82} & 21.77 & 59.73 & \textbf{69.20} \\
LongLive & 97.20 & 96.26 & 97.96 & 55.50 & \textbf{61.43} & \underline{68.66} \\
Rolling Forcing & 96.16 & 95.59 & 97.92 & \underline{88.20} & 57.05 & 67.73 \\
Video-Mirai & 96.41 & 96.21 & 96.93 & 36.43 & 56.54 & 63.84 \\
\midrule
\textbf{Ours} & 96.64 & 95.51 & 97.94 & \textbf{96.47} & \underline{60.32} & 68.07 \\
\bottomrule
\end{tabular}
\end{table}

\begin{table}[h]
\centering
\caption{Per-dimension VideoScore v1.1 scores.}
\label{tab:app_videoscore}
\small
\setlength{\tabcolsep}{4pt}
\begin{tabular}{lccccc}
\toprule
Method & Visual$\uparrow$ & Temporal$\uparrow$ & Dynamic$\uparrow$ & T2V Align.$\uparrow$ & Factual$\uparrow$ \\
\midrule
CausVid & 2.683 & 2.450 & 2.985 & 2.738 & 2.330 \\
Causal Forcing++ & 2.698 & 2.472 & 3.044 & 2.803 & 2.332 \\
Self Forcing & 2.908 & 2.731 & 3.242 & 3.046 & 2.597 \\
LongLive & \textbf{3.157} & \textbf{2.985} & \underline{3.345} & \underline{3.206} & \textbf{2.847} \\
Rolling Forcing & 2.843 & 2.608 & 3.095 & 2.876 & 2.452 \\
Video-Mirai & 2.679 & 2.461 & 3.187 & 2.752 & 2.441 \\
\midrule
\textbf{Ours} & \underline{3.117} & \underline{2.890} & \textbf{3.379} & \textbf{3.217} & \underline{2.742} \\
\bottomrule
\end{tabular}
\end{table}

\begin{table}[h]
\centering
\caption{MUSIQ scores of the first and last 5 seconds and the drift measure.}
\label{tab:app_musiq}
\small
\begin{tabular}{lccccc}
\toprule
Method & Head$\uparrow$ & Tail$\uparrow$ & Avg.$\uparrow$ & $\frac{1}{N}\sum_i|T_i-H_i|\downarrow$  \\
\midrule
CausVid & \underline{70.25} & 62.73 & 66.49 & 9.02  \\
Causal Forcing++ & 68.00 & 65.31 & 66.65 & 5.52  \\
Self Forcing & \textbf{71.06} & 66.61 & 68.84 & 6.18  \\
LongLive & 69.37 & \textbf{68.87} & \underline{69.12} & \textbf{3.43}  \\
Rolling Forcing & 69.48 & \underline{68.35} & 68.92 & 5.32  \\
Video-Mirai & 69.49 & 59.58 & 64.54 & 11.90  \\
\midrule
\textbf{Ours} & 70.23 & 68.04 & \textbf{69.14} & \underline{3.64} \\
\bottomrule
\end{tabular}
\end{table}

\end{document}